\documentclass{article}
\usepackage{iclr2027_conference,times}
\usepackage[hyphens]{url}
\usepackage{graphicx}
\usepackage{caption}
\usepackage{algorithm}
\usepackage{algorithmic}
\usepackage{amsmath,amssymb,amsthm,mathtools}
\usepackage{booktabs}
\usepackage{array}

\usepackage{hyperref}
\newcommand{\apptoclink}[2]{\hyperref[#1]{#2}}

\usepackage{tabularx}
\usepackage{placeins}
\usepackage{thmtools}
\usepackage{thm-restate}
\makeatletter
\@ifundefined{newcounteralias}{}{%
  \renewcommand\thmt@autorefsetup{\@xa\def\csname\thmt@envname autorefname\@xa\endcsname\@xa{\thmt@thmname}}%
}
\makeatother
\usepackage{tikz}
\usepackage{comment}
\usepackage{enumitem}
\usepackage{xcolor}
\usepackage[nameinlink,noabbrev,capitalize]{cleveref}

\hypersetup{
  colorlinks=true,
  citecolor=blue,
  linkcolor=blue,
  urlcolor=blue,
  pdfauthor={Ziang Gao, Pengqi Liu, Archer Yi Yang, Mouloud Belbahri, Jesse C. Cresswell, Masoud Asgharian},
  pdftitle={Tilted Mondrian: Triage and Mitigation of Fine-grained Disparities in Conformal Prediction},
  pdfsubject={Preprint},
  pdfkeywords={conformal prediction}
}

\newtheorem{theorem}{Theorem}

\theoremstyle{definition}
\newtheorem{definition}[theorem]{Definition}

\theoremstyle{plain}

\crefname{assumption}{Assumption}{Assumptions}
\Crefname{assumption}{Assumption}{Assumptions}
\crefname{proposition}{Proposition}{Propositions}
\Crefname{proposition}{Proposition}{Propositions}
\crefname{theorem}{Theorem}{Theorems}
\Crefname{theorem}{Theorem}{Theorems}
\crefname{lemma}{Lemma}{Lemmas}
\Crefname{lemma}{Lemma}{Lemmas}
\crefname{section}{Section}{Sections}
\Crefname{section}{Section}{Sections}
\crefname{appendix}{Appendix}{Appendices}
\Crefname{appendix}{Appendix}{Appendices}

\title{%
    \begin{minipage}{\textwidth}
        \begin{minipage}[t]{0.18\textwidth}
        
\definecolor{background}{RGB}{235,235,235}
\definecolor{frame}{RGB}{255,255,255}
\definecolor{surface}{RGB}{244,246,246}
\definecolor{yellow}{RGB}{228,206,87}
\definecolor{blue}{RGB}{0,39,127}
\definecolor{orange}{RGB}{213,85,48}
\definecolor{border}{RGB}{215,215,201}

\begin{tikzpicture}[
    scale=0.14,
    x=0.01cm,
    y=-0.01cm,
    line join=miter,
    line cap=butt
]

\fill[frame]
    (844,51) --
    (1634,839) --
    (844,1634) --
    (51,840) --
    cycle;

\fill[surface]
    (844,59) --
    (1626,839) --
    (844,1625) --
    (60,840) --
    cycle;

\fill[yellow]
    (112,790) --
    (656,790) --
    (656,250) --
    cycle;

\fill[blue]
    (1286,790) --
    (1582,790) --
    (1626,835) --
    (1291,1178) --
    cycle;

\fill[black]
    (374,799) rectangle (670,1125);

\fill[orange]
    (599,1382) --
    (654,1368) --
    (654,1435) --
    cycle;

\fill[black]
    (721,182) rectangle (962,202);

\fill[black]
    (120,788) rectangle (1583,812);

\fill[black]
    (654,255) rectangle (678,1432);

\fill[black]
    (1272,522) rectangle (1294,1168);

\fill[black]
    (606,1367) rectangle (1075,1385);

\fill[black]
    (374,799) rectangle (394,1145);

\draw[
    border,
    line width=0.4pt
]
    (844,59) --
    (1626,839) --
    (844,1625) --
    (60,840) --
    cycle;

\end{tikzpicture}

        \end{minipage}%
        \begin{minipage}[t]{0.78\textwidth}
        \vspace{-58pt}
            Tilted Mondrian: Triage and \\Mitigation of Fine-grained \\ Disparities in Conformal Prediction
        \end{minipage}
    \end{minipage}
}

\iclrfinalcopy

\author{
  Ziang Gao\thanks{Equal contribution. Alphabetically ordered by last name.}\ 
  \thanks{Corresponding author. Email: {\usefont{T1}{cmtt}{m}{n}ziang.gao@mail.mcgill.ca}.} \\
  McGill University\\
  Montreal, Canada\\
  {\small\usefont{T1}{cmtt}{m}{n}ziang.gao@mail.mcgill.ca}
  \And
  Pengqi Liu\footnotemark[1] \\
  McGill University\\
  Montreal, Canada\\
  {\small\usefont{T1}{cmtt}{m}{n}pengqi.liu@mail.mcgill.ca}
  \And
  Archer Yi Yang \\
  McGill University\\
  Montreal, Canada\\
  {\small\usefont{T1}{cmtt}{m}{n}archer.yang@mcgill.ca}
  \AND
  Mouloud Belbahri \\
  TD Insurance\\
  Montreal, Canada\\
  {\small\usefont{T1}{cmtt}{m}{n}mouloud.belbahri@td.com}
  \And
  Jesse C. Cresswell \\
  Layer 6 AI\\
  Toronto, Canada\\
  {\small\usefont{T1}{cmtt}{m}{n}jesse@layer6.ai}
  \And
  Masoud Asgharian \\
  McGill University\\
  Montreal, Canada\\
  {\small\usefont{T1}{cmtt}{m}{n}masoud.asgharian2@mcgill.ca}
}

\begin{document}

\maketitle
\lhead{Preprint}

\begin{abstract}
Fairness in conformal prediction is often pursued through Mondrian calibration, which provides finite-sample coverage guarantees separately across groups. These guarantees, however, are coarse: substantial disparities can remain within finer strata. This coarse-to-fine fairness issue raises two questions: how severe can such disparities be, and how can they be mitigated? Aggregate disparities can depend on how groups are distributed across strata. For severity assessment, we introduce propensity standardization to examine the effect of population composition,  and show how group composition can induce disagreement between aggregate and finer-stratum comparisons. For mitigation, we develop Tilted Mondrian to improve fairness compared to group-Mondrian at a finer level without additional model training. Tilted Mondrian retains groupwise coverage guarantees and mitigates coverage and prediction set size disparities within finer strata. Finally, we derive a method-agnostic bound that quantifies how much reduction in finer-stratum coverage disparity can be achieved. 
\end{abstract}

\section{Introduction}\label{sec:intro}

Conformal prediction (CP) \citep{vovk2005algorithmic} attaches prediction sets to fitted classifiers with finite-sample coverage guarantees. A natural way to address groupwise fairness is group-Mondrian calibration \citep{vovk2003mondrian}, which guarantees coverage across prespecified groups. Yet this guarantee is only group conditional. Specifically, substantial disparities may remain within finer strata of the groups. At a fixed coverage level, aggregate group comparisons can obscure behavior within strata. Equalized groupwise coverage guarantees can coexist with large disparities within finer strata, while for set size, group composition can alter the aggregate disparity. As a result, we view both as manifestations of the same coarse-to-fine problem. We first assess the impact of composition, that is, the distribution across finer strata within each group, by comparing set size disparities under the observed and standardized populations. We then adjust the nonconformity score to mitigate disparities within finer strata while preserving groupwise coverage. 

Given a CP coverage level, we study the question of how much ambiguity remains, as measured by prediction set size, and how it is distributed across populations. Coverage controls how often the true label is included, but not how many labels a user needs to consider in downstream decision-making. For example, prediction sets of size two and eight can both be valid at the same nominal level yet impose very different burdens on a downstream review process~\citep{cresswell2024human,cresswell2025disparate}. 

We summarize this burden by a set size curve, which records expected prediction set size as a function of coverage level. Group curves can mix difficulty with population composition. A group may receive smaller prediction sets overall simply because its members are concentrated in easier-to-predict subpopulations even when the within-subpopulation comparison goes in the opposite direction. Interpreting finer-grained disparities requires separating within-stratum differences from population composition. We therefore introduce propensity score standardization, which leaves the label predictor, nonconformity scores, calibration thresholds, prediction sets, and coverage guarantees unchanged while reaveraging set sizes over a standardized population. By placing groups on a standardized distribution, the set size curves assess how sensitive an aggregate difference is to population composition.

We next turn to score intervention when coarse groupwise coverage masks disparities across finer strata. For example, even when group-Mondrian calibration equalizes coverage across female and male groups, substantial disparities may persist within finer-grained strata, such as occupations. We introduce Tilted Mondrian CP, a score adjustment framework for mitigating fine-grained disparities while preserving groupwise coverage guarantees. When targeting coverage, it seeks to reduce stratum coverage disparities while measuring set size as an efficiency cost. The same framework can also target stratum set size disparities and track changes in coverage.

\begin{figure}[t]
    \centering
    \includegraphics[width=\textwidth]{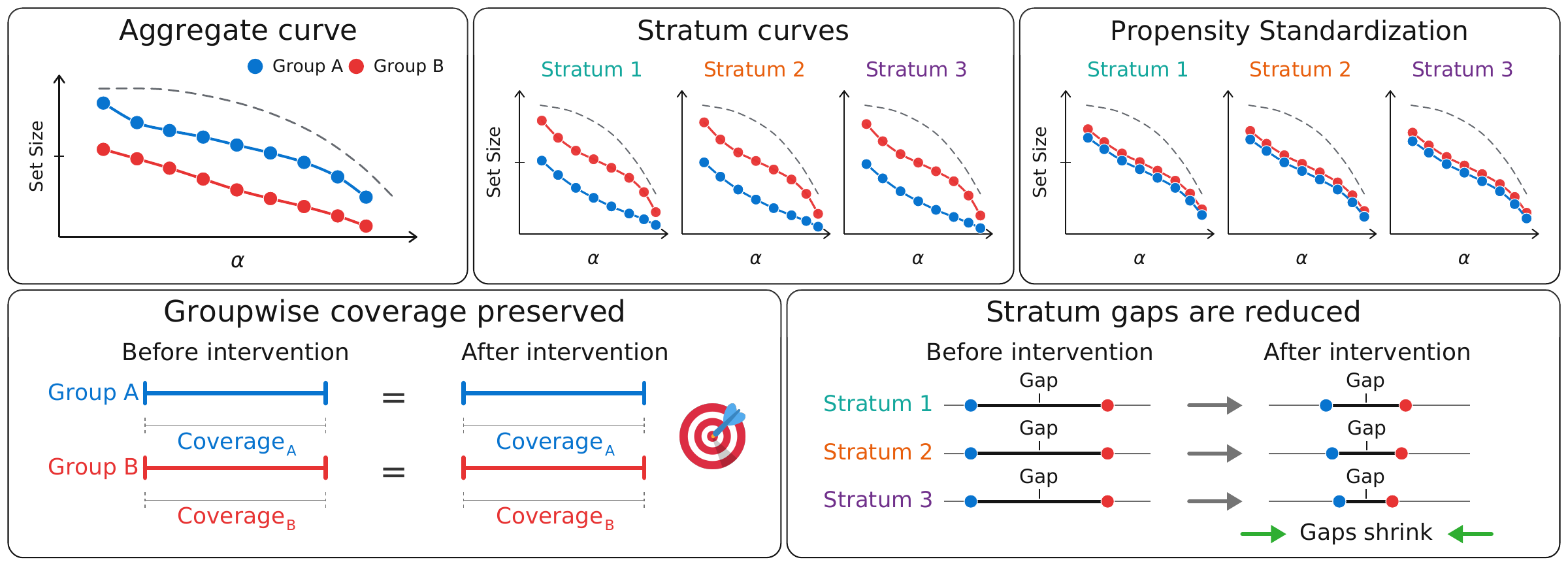}
    \caption{Top: Propensity standardization evaluates the aggregate--stratum set size disagreement. Bottom: Tilted Mondrian mitigates finer strata disparities and maintains groupwise coverage.}
    \label{fig:graphical_abstract}
    \vspace{-10pt}
\end{figure}

Our two main contributions, summarized in Figure~\ref{fig:graphical_abstract}, are as follows.
\begin{itemize}[leftmargin=*, itemsep=1.5pt, topsep=1pt]
\vspace{-3pt}
\item \textbf{Severity assessment of composition via propensity standardization.} 
We introduce propensity standardization that evaluates composition within finer stratum disparities, revealing how sensitive aggregate--stratum misalignment is to population composition.

\item \textbf{Tilted Mondrian and mitigation limits.} We develop Tilted Mondrian as a generic mitigation approach that can be instantiated at any Mondrian level and can target finer-stratum disparities in coverage and set size. For coverage mitigation, we derive a method-agnostic bound characterizing the mitigation limits under equalized group coverage.
\vspace{-3pt}
\end{itemize}

The remainder of this paper is organized as follows: Section~\ref{sec:related-work} reviews related work; Sections~\ref{sec:curves}--\ref{sec:standardization} develop the composition assessment; Section~\ref{sec:no-free-repair-intervention} studies Tilted Mondrian mitigation; Section~\ref{sec:empirical-protocol-evidence} presents the experimental evidence on assessment and mitigation; and Section~\ref{sec:limitations-conclusion} concludes. 

\section{Related work}\label{sec:related-work}

Conformal prediction gives distribution-free set-valued coverage under exchangeability \citep{vovk2005algorithmic,angelopoulos2023gentle}. %
Average set size is a standard efficiency measure for set-valued predictors. Methods such as APS, RAPS, and SAPS construct adaptive classification sets that aim to minimize average set size \citep{romano2020classification,angelopoulos2021uncertainty,huang2024conformal}. We formalize set size efficiency as a curve across miscoverage levels and read its height as residual ambiguity. Related work studies direct tradeoffs between set size and conditional coverage \citep{bach2025convex}, and connects CP set size to information theory \citep{correia2024information}. We study set size and use classical entropy–error relations \citep{fano1961transmission,cover2006elements,feder1994relations} to explain when large sets may be unavoidable.

Fairness-oriented CP studies equalized coverage, adaptively selected groups, demographic-parity quantiles, equal-opportunity coverage, counterfactual fairness, and domain-specific fair predictors \citep{romano2020malice,zhou2024conformal,vadlamani2025framework,liu2022conformalized,wang2023equal,guldogan2026counterfactual,lu2022fair}. Complementary work examines human decisions and substantive fairness when set-size burdens differ across groups \citep{cresswell2024human,cresswell2025disparate,liu2026beyond,gao2026burden}.  Our objective is more specific: a descriptive ambiguity comparison for a fixed predictor, fixed calibration rule, and stated evaluation population. 

We connect population imbalance to Simpson's paradox \citep{simpson1951interpretation}. %
Our standardization step draws on propensity reweighting, covariate balancing, and overlap populations \citep{rosenbaum1983central,imai2014covariate,li2018balancing}. Related methods pursue conditional coverage via multivalid or conditional calibration \citep{jung2022batch,gibbs2025conformal,bairaktari2025kandinsky}, or score transformations \citep{xie2024boosted,plassier2025rectifying}. Prior methods primarily ask how to achieve conditional coverage or set efficiency. Our objective is to mitigate stratum-level disparities under explicit groupwise coverage and stratum constraints.

\section{Aggregate Stratum Disagreement in Set Size Curves}\label{sec:curves}\label{sec:simpson}

Consider a multiclass problem with input $X$, label $Y\in\mathcal{Y}$ where $\mathcal{Y}$ is a finite label space of size $K\ge2$, and a nonconformity score $s(x,y)$, where smaller values indicate that $y$ is more compatible with $x$.  Given an exchangeable calibration sample $D_{\rm cal}=\{(X_i,Y_i)\}_{i=1}^{n_{\rm cal}}$, write $S_i=s(X_i,Y_i)$, $i=1,\ldots,n_{\rm cal}$, for the calibration scores. Split CP computes the empirical $(1-\alpha)$-quantile
\begin{equation}
    \hat q_\alpha=\operatorname{Quantile}\bigl(1-\alpha;\{S_i\}_{i=1}^{n_{\rm cal}}\cup\{\infty\}\bigr),
    \label{eq:conformal-quantile}
\end{equation}
equivalently the $\lceil(n_{\rm cal}+1)(1-\alpha)\rceil$-th order statistic of the augmented multiset. The calibrated prediction set at level $1-\alpha$ is
\begin{equation}
    \mathcal C_\alpha(x)=\{y\in\mathcal{Y}:s(x,y)\le \hat q_\alpha\}.
    \label{eq:conformal-set}
\end{equation}
Under exchangeability between the calibration and the test point, split CP guarantees
\begin{equation}
    \mathbb P\{Y_{\mathrm{test}}\in \mathcal C_\alpha(X_{\mathrm{test}})\}\ge 1-\alpha.
\end{equation}
Coverage controls true-label inclusion, but not set size.  
\begin{definition}[Set size curve]
Let \(M_\alpha(X)=|\mathcal{C}_\alpha(X)|\) denote the size of the prediction set. Given a population \(P\), the set size curve is \(V^P(\alpha):=\mathbb{E}_{X\sim P} M_\alpha(X)\), the expected set size at level \(\alpha\).
\end{definition}
Here, \(P\) indexes the population over which the post-calibration set sizes are averaged. Thus, \(V^P(\alpha)\) summarizes the CP set efficiency at each miscoverage level. We omit the superscript \(P\) when the population is clear from context.
A set size curve should be interpreted against intrinsic label ambiguity. Appendix~\ref{app:uncertainty-principle} discusses that CP sets cannot be arbitrarily small. When the label \(Y\) remains hard to predict from the input \(X\), CP must use larger prediction sets on average to maintain the target $1-\alpha$ coverage level. We now turn to the curves \(V(\alpha)\) and examine how the associated ambiguity burden of set size is distributed across groups and strata of \(X\).

Let $G\in\{a,b\}$ denote the comparison group, e.g. binary gender. The raw group curves are
\begin{equation}
    V_g(\alpha)=\mathbb{E}\{M_\alpha(X)\mid G=g\},
    \qquad g\in\{a,b\},
\end{equation}
and the raw aggregate gap is $\Delta(\alpha)=V_a(\alpha)-V_b(\alpha)$. The estimand $V_g(\alpha)$ indicates how many labels are delivered for each group $G=g$. It is a meaningful population summary, but its fairness interpretation depends on whether the observed mixture is the comparison population of interest.

To expose what the aggregate mixes, in addition to $G$, we consider another prespecified reporting variable $Z\in\mathcal{Z}$, where \(\mathcal{Z}\) is a finite set, such as an occupation, education bucket, or employment category. We refer to each group--stratum pair $(g,z)$ as a cell. Define
    \(V_{g,z}(\alpha)=\mathbb{E}\{M_\alpha(X)\mid G=g,Z=z\}\) and \(
    \Delta_z(\alpha)=V_{a,z}(\alpha)-V_{b,z}(\alpha).\)
Let $\pi_g(z)=\mathbb{P}(Z=z\mid G=g)$. For every $\alpha$, by the law of total expectation, one can obtain the set size gap decomposition
\begin{equation}
\Delta(\alpha)=\sum_{z\in\mathcal{Z}}\pi_a(z)\Delta_z(\alpha)
    +\sum_{z\in\mathcal{Z}}\{\pi_a(z)-\pi_b(z)\}V_{b,z}(\alpha).
\label{eq:curve-aggregation}
\end{equation}
If groups $a$ and $b$ have the same cell mixture $\pi_a(z)=\pi_b(z)$, then the aggregate gap $\Delta(\alpha)$ reduces to the weighted average of the stratum gaps $\sum_{z\in\mathcal{Z}}\pi_a(z)\Delta_z(\alpha)$, but this is not generally the case. The first term in \cref{eq:curve-aggregation} averages stratum gaps under group $a$'s mixture. The second isolates composition, which can make the aggregate gap $\Delta(\alpha)$ have the opposite sign from a stratum gap $\Delta_z(\alpha)$. This is an analog of Simpson's paradox \citep{simpson1951interpretation} realized in set sizes.%

When an aggregate group comparison differs from the corresponding finer-stratum comparisons, we refer to this as aggregate--stratum disagreement. On the supported \((z,\alpha)\) grid, define %
\begin{equation}
    \text{Reversal for }z\text{:}\qquad\Delta(\alpha)\Delta_z(\alpha)<0, \quad
    \text{Conflict strength for }z\text{:}\quad \Gamma_z(\alpha)=\left[-\Delta(\alpha)\Delta_z(\alpha)\right]_+,
    \label{eq:simpson-product-strength}
\end{equation}
respectively. We interpret a reversal jointly with conflict strength. A reversal records when aggregate and stratum gaps point in opposite directions, whereas conflict strength accounts for the magnitudes of both gaps, so a reversal with negligible conflict strength carries little weight. We use these quantities to report reversal counts and conflict strength in the experiments in Section~\ref{sec:empirical-protocol-evidence}.

From \Cref{eq:simpson-product-strength}, $\Gamma_z(\alpha)>0$ identifies an aggregate--stratum sign conflict. \Cref{eq:curve-aggregation} shows that such disagreement can arise from \(\pi_a \neq \pi_b\). The next section therefore reaverages the same \(M_\alpha(X)\) over a standardized comparison population to assess the severity of the composition effect.

\section{Propensity Score Standardization for Disparity Assessment}\label{sec:standardization}

The raw group curves average the set size over the group-specific population distributions. To separate this population mixture component from disparities in set size, we apply propensity score reweighting to place both groups on a standardized distribution. The fitted predictor and CP procedure remain fixed; only the population over which their set sizes are averaged is changed. 

Let \(H=H(X)\) denote the \textit{profile}, a prespecified vector of balancing variables defined as a function of \(X\), which could include probabilities from a classifier, margins, entropy, and top-label indicators. The profile \(H\) excludes the group variable \(G\), realized labels, realized set sizes, calibration quantities, and final curve summaries. The profile \(H\) defines the target population to be standardized, whereas \(Z\) defines the finer strata used to assess aggregate--stratum alignment. Including \(Z\) among the balancing variables would instead standardize the joint \((H,Z)\) distribution and thus define a different estimand. The balancing variables are fixed before the final curve analysis.

Let $P_H$ be the marginal distribution of $H$, and define the group propensity score~\citep{li2018balancing}
\begin{equation}
    e(h)=\mathbb{P}(G=a\mid H=h).
\end{equation}
The propensity score describes how profiles are distributed across the two groups and determines the standardization weights. Comparing the raw and standardized curves then assesses how the aggregate--stratum relationship changes under the standardized population.
One option is to standardize both groups to the marginal profile distribution $P_H$. Up to normalizing constants,
\begin{equation}
    \mathbb{P}_{H\mid G=a}\propto e(h)P_H,\qquad \mathbb{P}_{H\mid G=b}\propto \{1-e(h)\}P_H.
\end{equation}
Standardizing both groups to \(P_H\) applies the inverse propensity weights \(1/e(H)\) for group \(a\) and \(1/\{1-e(H)\}\) for group \(b\). This inverse probability weighting (IPW) retains the marginal profile population, but the weights can become unstable when \(e(H)\) approaches \(0\) or \(1\), making the comparison sensitive to limited overlap. %
For the main analysis, we apply propensity reweighting
\begin{equation}
    Q_H=\frac{e(H)\{1-e(H)\}P_H}{\mathbb{E}[e(H)\{1-e(H)\}]}\quad \text{provided that } \mathbb{E}[e(H)\{1-e(H)\}]>0\ .
    \label{eq:overlap-target}
\end{equation}
The corresponding weights are \(w_a(H)=1-e(H)\) and \(w_b(H)=e(H)\), which map both groups to the standardized target distribution \(Q_H\). 
Equivalently, \(H\) and \(G\) are independent in the weighted population. The target emphasizes profiles represented in both groups and downweights profiles with small overlap. The propensity reweighting method also avoids the extreme weighting that can arise under classical IPW. Thus, we identify \(Q_H\) as the target population for the standardized curves. %
For group $g \in \{a,b\}$, define the standardized curve by
\begin{equation}
    \tilde V_g(\alpha)=\mathbb{E}_{Q_H}\,\mathbb{E}\{M_\alpha(X)\mid G=g,H\}.
    \label{eq:standardized-curve}
\end{equation}
The raw curve \(V_g(\alpha)\) averages over the current data distribution, whereas \(\tilde V_g(\alpha)\) averages over the standardized distribution \(Q_H\). Empirically, the standardized curve \(\tilde V_g(\alpha)\) is estimated by a normalized H\'ajek mean with propensity weights \citep{hajek1971comment}.
The same propensity reweighting induces a group-specific distribution over the stratum variable \(Z\):
\begin{equation}
    \tilde\pi_g(z)=\frac{\mathbb{E}\{\mathbf{1}(G=g)w_g(H)\mathbf{1}(Z=z)\}}{\mathbb{E}\{\mathbf{1}(G=g)w_g(H)\}}=\mathbb{E}_{Q_H}[\mathbb{P}(Z=z \mid G=g,H)].
\label{eq:q-induced-cell-law}
\end{equation}
For any cell \((g,z)\) with $\tilde\pi_g(z)>0$, define the standardized stratum curve 
\begin{equation}
\tilde V_{g,z}(\alpha)=\frac{1}{\tilde\pi_g(z)}\mathbb{E}_{Q_H}[\mathbb{E}[\mathbf{1}(Z=z)M_\alpha(X)\mid G=g,H]].
\end{equation}
Then $\tilde V_g(\alpha)=\sum_z\tilde\pi_g(z)\tilde V_{g,z}(\alpha)$. Define the standardized aggregate and stratum gaps by
\begin{equation}
\tilde\Delta(\alpha)=\tilde V_a(\alpha)-\tilde V_b(\alpha),
\qquad
\tilde\Delta_z(\alpha)=\tilde V_{a,z}(\alpha)-\tilde V_{b,z}(\alpha).
\end{equation}
The next proposition shows that both reweighted groups follow the standardized distribution \(Q_H\). \emph{Complete proofs for all theoretical results are given in Appendix\hyperref[proof:common-profile-balance]{\text{~\ref{app:proofs}.}}}

\begin{restatable}[Propensity score reweighting]{proposition}{commonprofilebalanceprop}\label{prop:common-profile-balance}
For every integrable function $f$ of $H$,
\begin{equation}
\frac{\mathbb{E}\{f(H)\mathbf{1}(G=a)(1-e(H))\}}{\mathbb{E}\{\mathbf{1}(G=a)(1-e(H))\}} = \frac{\mathbb{E}\{f(H)\mathbf{1}(G=b)e(H)\}}{\mathbb{E}\{\mathbf{1}(G=b)e(H)\}} = \mathbb{E}_{Q_H} [f(H)].
\end{equation}
\end{restatable}
\noindent\textbf{Take-away.}
Propensity standardization puts both groups on the same population distribution, making an assessment of set sizes more comparable.

To relate the standardized aggregates to stratum gaps, let \(\mathbb{E}_Z\) and \(\operatorname{Var}_Z\) denote expectation and variance under the symmetric raw distribution \(\{\pi_a(z)+\pi_b(z)\}/2\), and let \(\mathbb{E}_{\tilde Z}\) and \(\operatorname{Var}_{\tilde Z}\) denote the corresponding quantities under \(\{\tilde\pi_a(z)+\tilde\pi_b(z)\}/2\). Define
$R(\alpha)=\Delta(\alpha)-\mathbb{E}_Z\Delta_Z(\alpha)$ and $\tilde R(\alpha)=\tilde\Delta(\alpha)-\mathbb{E}_{\tilde Z}\tilde\Delta_Z(\alpha).$ 
Thus, \(R(\alpha)\) is the composition effect in the raw aggregate gap, and \(\tilde R(\alpha)\) is its standardized counterpart. %
\Cref{eq:standardized-residual,eq:standardized-residual-tv-bound} give the explicit form of \(\tilde R(\alpha)\) and show that it is controlled by the difference between the two standardized distributions \(\tilde\pi_a\) and \(\tilde\pi_b\).%

To quantify aggregate--stratum alignment, we %
formulate the following alignment losses
\begin{equation}
\mathcal{A}_{\rm raw}(\alpha)
=\mathbb{E}_Z\{\Delta(\alpha)-\Delta_Z(\alpha)\}^2, \qquad
\mathcal{A}_{\rm std}(\alpha)
=\mathbb{E}_{\tilde Z}\{\tilde\Delta(\alpha)-\tilde\Delta_Z(\alpha)\}^2.
\label{eq:alignment-loss-budget}
\end{equation}
A smaller \(\mathcal{A}\) means that the aggregate gap is closer, on average, to the corresponding stratum gaps.

For mass functions $p,q$ on the stratum grid, we define the total variation $\operatorname{TV}(p,q):=\frac12\sum_z|p(z)-q(z)|$. Although propensity reweighting equalizes the distribution of \(H\), the
conditional distribution of \(Z\) given \(H\) may still differ across groups. We quantify this remaining difference by%
\begin{equation}
\eta_H=\mathbb E_{Q_H}
\!\left[
\operatorname{TV}\!\left\{
\mathbb P(Z\in\cdot\mid G=a,H),
\mathbb P(Z\in\cdot\mid G=b,H)
\right\}
\right].
\end{equation}
Thus, \(\eta_H\) computes the average conditional difference between the two stratum distributions at the same profile \(H\). In particular, \(\eta_H\) vanishes exactly when \(\mathbb{P}(Z\in\cdot\mid G=a,H=h)=\mathbb{P}(Z\in\cdot\mid G=b,H=h)\) for \(Q_H\)-almost every \(h\). Next, define
\begin{equation}
\Delta_{h,z}(\alpha):=
\mathbb E\{M_\alpha(X)\mid G=a,H=h,Z=z\}-\mathbb E\{M_\alpha(X)\mid G=b,H=h,Z=z\}.
\label{eq:profile-cell-gap}
\end{equation}
This quantity compares the set size between groups at a fixed profile $h$ and stratum $z$. When \(\eta_H\) vanishes, we say that composition dominance holds at level $\alpha$ if
\begin{equation}
\label{eq:composition-dominance}
\operatorname*{ess\,sup}_{h,z}|\Delta_{h,z}(\alpha)-\mathbb{E}\{\Delta_{H,Z}(\alpha)\}|<|R(\alpha)|
\ ,
\end{equation}
where the essential supremum and expectation are taken under the joint distribution induced by $H\sim Q_H$ and the conditional distribution of $Z\mid H$.

\begin{restatable}[Alignment under standardization]{theorem}{alignmentlossthm}\label{thm:alignment-loss}
Given a fixed level $\alpha$, the propensity score standardization satisfies
\(
\operatorname{TV}(\tilde\pi_a,\tilde\pi_b)\le\eta_H,
|\tilde R(\alpha)|\le K\eta_H.
\)
Suppose $\eta_H$ vanishes, equivalently
\(
\mathbb P(Z\in\cdot\mid G=a,H=h)=\mathbb P(Z\in\cdot\mid G=b,H=h)
\) for $Q_H$-almost every $h$, then $\tilde R(\alpha)=0$ and
\begin{equation}
\mathcal A_{\rm std}(\alpha)
\le
\operatorname{Var}
\{\Delta_{H,Z}(\alpha)\}\ .
\end{equation}
Furthermore, suppose the composition dominance holds (\Cref{eq:composition-dominance}), then
\begin{equation}
\mathcal A_{\rm std}(\alpha)<\mathcal A_{\rm raw}(\alpha)\ .
\end{equation}
In particular, suppose $\Delta_{h,z}(\alpha)$ is constant for every \(h,z\) pair, then $\mathcal A_{\rm std}(\alpha)=0$.
\end{restatable}

\begin{figure*}[t]
\begin{minipage}[t]{\textwidth}
\refstepcounter{algorithm}\label{alg:main-curve-reporting}
\small
\raggedright
\hrule
\vspace{0.7ex}

\noindent\textbf{Algorithm \thealgorithm:} Propensity Standardization

\noindent\textbf{Input:} $\mathcal D_{\rm prop},\mathcal D_{\rm cal},\mathcal D_{\rm test}$; group $G$, stratum $Z$, profile rule $H$; grid $\boldsymbol{\alpha}$.

\small
\setlength{\algorithmicindent}{0.8em}
\begin{algorithmic}[1]
\STATE Construct \(H\) from the frozen outputs. On \(\mathcal D_{\mathrm{prop}}\), fit $\hat e(h)=\hat{\mathbb P}(G=a\mid H=h)$.
\STATE Define standardized target $Q_H$ by \cref{eq:overlap-target}; freeze weights/support.
\STATE On $\mathcal D_{\rm cal}$, perform split CP for all $\alpha$.
\STATE On $\mathcal D_{\rm test}$, set $\hat e_i=\hat e(H_i)$, $w_{a,i}\propto1-\hat e_i$, and $w_{b,i}\propto\hat e_i$; bound/normalize by group.
\FOR{$\alpha\in\boldsymbol{\alpha}$}
    \STATE Form each $\mathcal C_{\alpha,i}$; set $M_{\alpha,i}=|\mathcal C_{\alpha,i}|$.
    \STATE From $M_{\alpha,i}$/cells, compute raw (unweighted) and standardized (\cref{eq:standardized-curve}) group/cell curves.
    \STATE Compare aggregate/strata gap signs; record reversals/strengths (\cref{eq:simpson-product-strength}).
\ENDFOR
\RETURN raw/standardized reports from fixed prediction sets.
\end{algorithmic}
\vspace{0.3ex}
\hrule
\end{minipage}
\vspace{-5pt}
\end{figure*}
\noindent\textbf{Take-away.}
Comparing the current population with a balanced population reveals how severely the composition affects the aggregate--stratum misalignment. A larger reduction indicates that propensity standardization removes a more severe composition effect. The remaining misalignment reflects heterogeneity across the strata of the data.

Operationally, in Theorem~\ref{thm:alignment-loss}, comparing \(\mathcal A_{\rm raw}\) and \(\mathcal A_{\rm std}\) quantifies how sensitive aggregate--stratum misalignment is to the population used for averaging. A large reduction indicates that the observed raw aggregate is strongly shaped by the population mixture of the observed dataset. Algorithm~\ref{alg:main-curve-reporting} summarizes the propensity standardization procedure.  Table~\ref{tab:alignment-decomposition} in Appendix~\ref{app:experiments} displays the decomposition with empirical results. Our analysis here is descriptive, and we make no causal claims.

\begin{figure}[t]
    \centering
    \includegraphics[width=\linewidth]{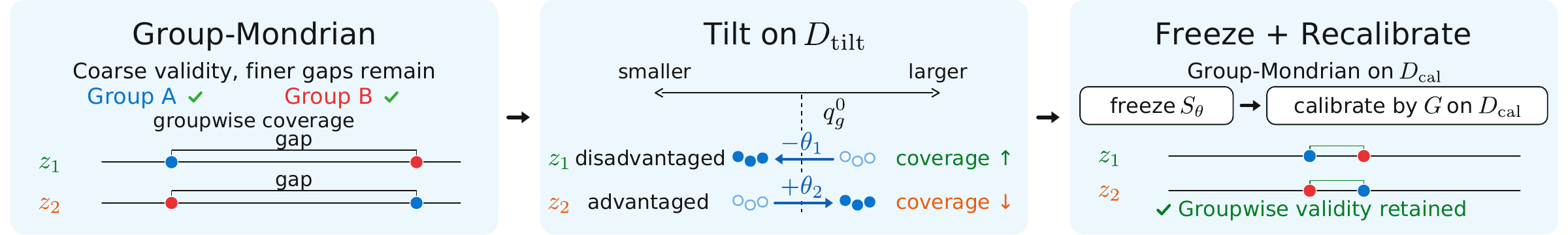}
    \vspace{-10pt}
    \caption{Tilted Mondrian mechanism for finer strata disparity mitigation.}
    \label{fig:tm-mechanism}
    \vspace{-10pt}
\end{figure}
\section{Tilted Mondrian and Mitigation Limits}\label{sec:no-free-repair-intervention}
\vspace{-4pt}
\subsection{Tilted Mondrian CP}
\vspace{-4pt}

We now turn to mitigation and propose Tilted Mondrian CP, a generic framework built on a chosen Mondrian calibration level. The intervention acts through a learned adjustment of the nonconformity score. The algorithm shifts coverage across finer strata to reduce disparities while maintaining groupwise coverage guarantees. %
We identify coverage, the reliability quantity directly guaranteed by CP, as the primary mitigation target.

For each $z\in\mathcal Z$, let $c^0_{g,z}$ denote the baseline cellwise coverage for group \(g\) and stratum \(z\). Let $g^+(z)$ and $g^-(z)$ be the groups with higher and lower baseline coverage, respectively, and define the baseline gap \(d_z=c^0_{g^+(z),z}-c^0_{g^-(z),z}\). Tilted Mondrian uses nonnegative score adjustments \(\theta^+_z\) and \(\theta^-_z\). It raises scores in \(g^+(z)\) by \(\theta_z^+\) and lowers scores in \(g^-(z)\) by \(\theta_z^-\). Let \(S_\theta\) denote the transformed score, with \(\theta=0\) as the baseline. Let \(c^\theta_{g,z}\) denote the resulting cellwise coverage, and define \(d_z^\theta=|c^\theta_{g^+(z),z}-c^\theta_{g^-(z),z}|\) and \(m_z^\theta=d_z-d_z^\theta\). Here, \(d_z^\theta\) is the coverage gap between the two groups in stratum \(z\) after the adjustment, and \(m_z^\theta\) is its reduction from the baseline gap \(d_z\); \(m_z^\theta>0\) indicates an improvement. A score shift is feasible only if the coverage loss in every initially advantaged cell is at most \(\epsilon\), coverage does not decrease in any initially disadvantaged cell, and stratum coverage gaps are non-increasing. During the tilt procedure, we preserve each group's baseline empirical coverage. We assume that \(\mathcal Z\) partitions each group population and that \(\pi_{g^-(z)}(z)>0\) for all \(z\in\mathcal Z\) (recall $\pi_g(z)=\mathbb P(Z=z\mid G=g)$). %
Figure~\ref{fig:tm-mechanism} illustrates the score adjustment and subsequent
groupwise recalibration.  Algorithm~\ref{alg:main-directional-shifting} summarizes Tilted Mondrian CP.

\begin{restatable}[Validity of Tilted Mondrian]{theorem}{scorefixvaliditythm}\label{prop:score-fix-validity}
Let \(\theta\) be selected on a tilt split independent of calibration. If group-Mondrian split conformal calibration is applied to \(S_\theta\) to generate prediction sets $\mathcal C_{\alpha,\theta}(X)$ as in \Cref{eq:conformal-set}, and calibration and test examples are exchangeable conditional on \(G=g\), then \(\mathbb P\{Y\in \mathcal C_{\alpha,\theta}(X)\mid G=g\}\ge 1-\alpha\) for each group \(g\in\{a,b\}\).
\vspace{-4pt}
\end{restatable}
\noindent\textbf{Take-away.}
Tilted Mondrian maintains the group-Mondrian CP coverage guarantee. %
The cellwise post-processing adjustment of Tilted Mondrian also allows improvements to occur jointly across strata. When the same group has lower coverage in one cell and higher coverage in another, the two gaps can be reduced together. Let \(q_g^0\) denote the baseline threshold for group \(g\). The following theorem formalizes this case.

\begin{restatable}[Local mitigation]{theorem}{localmitigationprop}
\label{prop:local-mitigation}
Assume $g=g^-(z_i)=g^+(z_j)$ for supported \(z_i,z_j\) with $d_{z_i},d_{z_j}>0$ and $\epsilon>0$. Suppose the population cellwise nonconformity score distributions admit densities that are locally positive at the threshold \(q_g^0\). Then the Tilted Mondrian family admits a nonzero feasible tilt \(\theta\) preserving group coverage such that
$d_{z_i}^\theta<d_{z_i}$ and $d_{z_j}^\theta<d_{z_j}$, equivalently
$m_{z_i}^\theta>0$ and $m_{z_j}^\theta>0$, with all other stratum gaps unchanged.
\vspace{-5pt}
\end{restatable}
\noindent\textbf{Take-away.}
Under these conditions, Tilted Mondrian can strictly improve group-Mondrian by reducing coverage disparities without changing group coverage.
\begin{figure*}[t]
\begin{minipage}[t]{\textwidth}
\refstepcounter{algorithm}\label{alg:main-directional-shifting}
\small
\raggedright
\hrule
\vspace{0.7ex}

\noindent\textbf{Algorithm \thealgorithm:} Tilted Mondrian score adjustment

\noindent\textbf{Input:} Baseline nonconformity score $s$; $\mathcal D_{\rm tilt},\mathcal D_{\rm cal},\mathcal D_{\rm test}$; $G,Z$; error level $\alpha$; tolerance $\epsilon$.

\small
\setlength{\algorithmicindent}{0.9em}
\begin{algorithmic}[1]
\STATE On $\mathcal D_{\rm tilt}$, compute baseline group-Mondrian thresholds; record group coverage levels and cellwise coverage $c_{g,z}^0$.
\STATE For each $z$, let $g^+(z)\text{/}g^-(z)$ be the higher/lower-coverage groups and $d_z$ their gap.
\FOR{each deterministic candidate shift $\theta$}
    \STATE Shift candidate scores by $+\theta_z^+$ in $g^+(z)$, $-\theta_z^-$ in $g^-(z)$; set $z=y$ for $Z=Y$, else $z=Z(x)$.
    \STATE Reselect group thresholds on \(\mathcal D_{\rm tilt}\) to keep group coverage levels; compute $c_{g,z}^\theta,d_z^\theta$.
    \STATE Set gap reduction $m_z^\theta=d_z-d_z^\theta$.
    \STATE Retain $\theta$ iff \(c^0_{g^+(z),z}-c^\theta_{g^+(z),z}\le\epsilon,
    c^\theta_{g^-(z),z}\ge c^0_{g^-(z),z}\), and $d_z^\theta\le d_z$ for all $z$.
\ENDFOR
\STATE Select in order: min worst gap; min mean gap; max total mitigation; min shift norm;
and, only for remaining ties, minimum tilt split average set size.
\STATE Freeze score; group-Mondrian recalibrate on $\mathcal D_{\rm cal}$; evaluate once on $\mathcal D_{\rm test}$.
\RETURN selected score and test report.
\end{algorithmic}

\vspace{0.3ex}
\hrule
\end{minipage}
\vspace{-10pt}
\end{figure*}

When finer-stratum fairness and coverage guarantees are themselves the objective, Mondrian calibration can be performed separately within each cell, rather than just at the group level. However, when working with a finite calibration set that becomes finely partitioned, cell-wise coverage suffers from high variance \citep{angelopoulos2023gentle}. Our setting is different. We ask whether disparities within \(Z\) can be mitigated while calibration remains only at the coarser group level \(G\). Tilted Mondrian produces thresholds at the group level after score shifting, avoiding fine partitions of the calibration data, and mitigates disparities at the cell level with score adjustments. Tables~\ref{tab:joint-mondrian-intervention-full} and~\ref{tab:joint-mondrian-rmad-full} in Appendix~\ref{app:tilted-mondrian} compare Tilted Mondrian to the alternative: partitioning at the cell level and applying Mondrian for cell-level thresholds. Overall, Tilted Mondrian offers a practical %
mechanism for fairness improvement beyond the group level.

Tilted Mondrian can also target set size disparities. Let \(\bar M^0_{g,z}\) denote the baseline mean prediction set size and \(L^0_{z,S}=|\bar M^0_{a,z}-\bar M^0_{b,z}|\); define \(L^\theta_{z,S}\) and \(m^\theta_{z,S}=L^0_{z,S}-L^\theta_{z,S}\) analogously after the mitigation. Tilted Mondrian raises nonconformity scores in the group with the larger baseline mean set size and lowers them in the group with the smaller mean set size. Candidate tilts preserve group coverage level on \(\mathcal D_{\rm tilt}\), keep set size gaps and larger-side means non-increasing, and limit smaller-side increases. The adjusted score is frozen before group-Mondrian calibration on \(\mathcal D_{\rm cal}\). Details and coverage trade-offs about Tilted Mondrian for set size are provided in Appendix~\ref{app:alternative-repair-target}.
We also considered a covariate-adaptive alternative inspired by boosted CP \citep{xie2024boosted}. Tilted Mondrian often achieves larger worst-gap reductions and lower mean cell disparities than the boosted variant (Appendix~\ref{app:boosted-tilted-mondrian}). The remaining question is how much local improvement is possible when group coverage and cellwise harm are held fixed.

\subsection{A Mitigation Limit under Fixed Group Coverage}

Tilted Mondrian leaves the cell composition $\pi_g(z)$ unchanged and instead modifies scores to balance cell-level coverage. Because group coverage is the $\pi_g$-weighted average of cell coverage, preserving group coverage requires $\sum_{z\in\mathcal Z}\pi_g(z) (c_{g,z}^{\theta}-c_{g,z}^{0})=0$. Consequently, a coverage gain in one cell must be offset elsewhere within the same group, motivating the following mitigation limit.

\begin{restatable}[Mitigation bound]{theorem}{nofreerepairbudgetthm}\label{thm:no-free-repair-budget}
Consider a class of disparity mitigation methods satisfying the same constraints as Tilted Mondrian: they preserve each group's baseline coverage, lose no more than $\epsilon$ coverage for each advantaged cell, do not reduce coverage in any disadvantaged cell, and do not increase stratum gaps. Assume \(\pi_{g^-(z)}(z)>0\) for every \(z\in\mathcal{Z}\). Then
\begin{equation}
    \max_{z\in\mathcal Z} m_z^\theta
    \le
    \max_{z\in\mathcal Z}
    \min\left\{
        d_z,
        \epsilon\left(
        1+
        \frac{\sum_{r:g^+(r)=g^-(z)}\pi_{g^-(z)}(r)}{\pi_{g^-(z)}(z)}
        \right)
    \right\}.
\end{equation}
\end{restatable}
\noindent\textbf{Take-away.}
Disparities can be reduced by giving up some coverage in advantaged strata. The bound quantifies the maximum improvement achievable by any disparity mitigation strategy satisfying these constraints, including Tilted Mondrian.
\vspace{-3pt}
\section{Empirical evidence}\label{sec:empirical-protocol-evidence}
\vspace{-3pt}
We examine propensity standardization for aligning aggregate set size reports with finer-stratum and Tilted Mondrian for reducing finer-stratum disparities in coverage and set size. Groupwise coverage guarantees are retained throughout the mitigation procedures.

\noindent\textbf{Datasets and tasks.}
The four real-world experimental settings are Bias in Bios occupation prediction with gender groups \(G\) and occupation strata \(Z\) \citep{dearteaga2019bias}; Folktables occupation prediction with gender groups \(G\) and education strata \(Z\) across five ACS years \citep{ding2021retiring}; Folktables income quintile prediction with race groups \(G\) and education strata \(Z\); ATUS time-use activity prediction with gender groups \(G\) and employment strata \(Z\) \citep{bls2024atus}. We evaluate propensity standardization on set size disparities, and Tilted Mondrian mitigation of finer-stratum coverage disparities, along with its set size variant.

\noindent\textbf{Models and scores.}
We begin with training, calibration, and test data splits. %
$\mathcal D_{\rm train}$ fits DistilBERT \citep{sanh2019distilbert} for Bias in Bios, year-specific histogram gradient boosting for Folktables, and regularized histogram gradient boosting for ATUS. We test with four score functions based on model probabilities $\hat p(x)$: SIMPLE, $s(x,y)=1-\hat p_y(x)$, APS \citep{romano2020classification}, RAPS \citep{angelopoulos2021uncertainty}, and SAPS \citep{huang2024conformal}. The calibration data is divided into two disjoint parts; one remains \(\mathcal D_{\rm cal}\), and the other is used as \(\mathcal D_{\rm prop}\) or \(\mathcal D_{\rm tilt}\). $\mathcal D_{\rm prop}$ fits propensity weights on profile $H$ using a feed-forward network for Bias in Bios, or logistic regression for other datasets. $\mathcal D_{\rm tilt}$ selects the score adjustments $\theta_z$. $\mathcal D_{\rm cal}$ calibrates the group-Mondrian thresholds. Finally, $\mathcal D_{\rm test}$ is used for evaluation. Each split is random from the overall pool of data, making data in each set exchangeable. Appendix~\ref{app:experiments} gives the detailed experiment protocols and model specifications.

\noindent\textbf{Metrics.}
For propensity standardization, Table~\ref{tab:empirical-evidence} reports the number of reversals, high-conflict reversals (conflict strength \(>0.020\)), and conflict strength from \Cref{eq:simpson-product-strength}. %
The normalized effective sample size (ESS) \citep{kish1992weighting} is \((\sum_i w_i)^2/\sum_i w_i^2\), divided by $\vert \mathcal{D}_{\rm{test}}\vert$. 

For coverage disparity mitigation with Tilted Mondrian, Table~\ref{tab:score-intervention-simple}
reports the worst stratum coverage gap before and after adjustment,
\(L^0_{\max}=\max_z d_z\) and
\(L^\theta_{\max}=\max_z d_z^\theta\), together with the largest
stratum gap reduction \(\max_z m_z^\theta\). Harm is the largest coverage loss among initially advantaged cells using tolerance \(\epsilon=0.01\), and
\(\Delta|C|\) is the change in average test set size. %
Similarly, for set size disparity mitigation, let \(L^0_{z,S}\) and \(L^\theta_{z,S}\) denote the stratum set size gaps before and after adjustment, and define \(m^\theta_{z,S}=L^0_{z,S}-L^\theta_{z,S}\). Table~\ref{tab:set-size-intervention-simple} reports \(L^0_{\max,S}=\max_z L^0_{z,S}\), \(L^\theta_{\max,S}=\max_z L^\theta_{z,S}\), and the largest stratum gap reduction \(\max_z m^\theta_{z,S}\). Harm is the largest set size increase among initially smaller set cells using tolerance \(\epsilon=0.02\), and \(\Delta\mathrm{Cov}\) is the change in coverage.

\vspace{-3pt}
\subsection{Evidence for Propensity Standardization}
\vspace{-3pt}

The experiments quantify the composition sensitivity of prediction sets by examining the relevant metrics before and after propensity standardization is applied. %
Reductions in conflict strength reflect better aggregate--stratum alignment, while the ESS fraction records how much comparable support remains. Table~\ref{tab:empirical-evidence} shows that standardization reduces conflict strength across the reported settings and usually reduces the number of reversals, revealing a strong composition effect in the observed data. For Bias in Bios and ATUS, conflict strength decreases greatly across all score variants. The Folktables occupation task shows that high conflict strength reversals are nearly eliminated in every year. These reductions reflect Theorem~\ref{thm:alignment-loss} and imply that raw aggregate--stratum disagreement is strongly sensitive to the composition effect, so raw set-size curves should be interpreted accordingly. %
Table~\ref{tab:alignment-decomposition} in Appendix~\ref{app:experiments} reports the corresponding alignment loss decomposition. Variability across data splits is examined in Table~\ref{tab:reporting-resplit-robustness} in Appendix~\ref{app:split-robustness}.

\begin{table}[!t]
\caption{Metrics before and after propensity standardization is applied. 
Top: Bias in Bios/ATUS across score functions $s(x, y)$. Bottom: Folktables with SIMPLE. %
}
\vspace{-4pt}
\label{tab:empirical-evidence}
\label{tab:score-specific-evidence}
\label{tab:folktables-cross-year}
\centering
\small
\begingroup
\renewcommand{\arraystretch}{1.0}

\begin{minipage}[t]{0.49\textwidth}
\centering
\setlength{\tabcolsep}{1pt}
\begin{tabular*}{\linewidth}{@{\extracolsep{\fill}}lccc@{}}
\toprule
\multicolumn{4}{c}{\emph{Bias in Bios}} \\
\cmidrule(lr){1-4}
Score
& Reversals
& \begin{tabular}[c]{@{}c@{}}Conflict strength\end{tabular}
& ESS \\
\midrule
SIMPLE & $20\to2$ & $0.101\to0.006$ & $0.833$ \\
APS & $41\to7$ & $0.120\to0.012$ & $0.833$ \\
RAPS   & $56\to1$ & $0.151\to0.008$ & $0.833$ \\
SAPS   & $38\to0$ & $0.131\to0.005$ & $0.833$ \\
\bottomrule
\end{tabular*}
\end{minipage}
\hfill
\begin{minipage}[t]{0.49\textwidth}
\centering
\setlength{\tabcolsep}{1pt}
\begin{tabular*}{\linewidth}{@{\extracolsep{\fill}}lccc@{}}
\toprule
\multicolumn{4}{c}{\emph{ATUS}} \\
\cmidrule(lr){1-4}
Score
& Reversals
& \begin{tabular}[c]{@{}c@{}}Conflict strength\end{tabular}
& ESS \\
\midrule
SIMPLE & $36\to10$ & $0.464\to0.003$ & $0.982$ \\
APS    & $72\to51$ & $1.635\to0.125$ & $0.982$ \\
RAPS   & $72\to54$ & $1.228\to0.086$ & $0.982$ \\
SAPS   & $69\to45$ & $0.039\to0.007$ & $0.982$ \\
\bottomrule
\end{tabular*}
\end{minipage}

\vspace{5pt}

\begin{minipage}[t]{0.5\textwidth}
\centering
\setlength{\tabcolsep}{3pt}
\begin{tabular}{@{}l@{\hspace{6pt}}c@{\hspace{4pt}}c@{\hspace{4pt}}c@{\hspace{6pt}}c@{}}
\toprule
\multicolumn{5}{c}{\emph{Occupation task} (SIMPLE)} \\
\cmidrule(lr){1-5}
Year
& Reversals
& \begin{tabular}[c]{@{}c@{}}High conflict\end{tabular}
& \begin{tabular}[c]{@{}c@{}}Strength\end{tabular}
& \begin{tabular}[c]{@{}c@{}}ESS\end{tabular} \\
\midrule
2014 & $182{\to}114$ & $94\to1$ & $9.378{\to}0.092$ & $0.666$ \\
2015 & $182{\to}160$ & $87\to0$ & $7.809{\to}0.319$ & $0.611$ \\
2016 & $191{\to}137$ & $80\to0$ & $7.751{\to}0.111$ & $0.702$ \\
2017 & $192{\to}102$ & $70\to0$ & $6.822{\to}0.544$ & $0.629$ \\
2018 & $183{\to}101$ & $69\to0$ & $7.904{\to}0.543$ & $0.621$ \\
\bottomrule
\end{tabular}
\end{minipage}\hfill
\begin{minipage}[t]{0.5\textwidth}
\centering
\setlength{\tabcolsep}{2pt}
\begin{tabular}{@{}l@{\hspace{6pt}}c@{\hspace{4pt}}c@{\hspace{4pt}}c@{\hspace{6pt}}c@{}}
\toprule
\multicolumn{5}{c}{\emph{Income quintile task} (SIMPLE)} \\
\cmidrule(lr){1-5}
Year
& Reversals
& \begin{tabular}[c]{@{}c@{}}High conflict\end{tabular}
& \begin{tabular}[c]{@{}c@{}}Strength\end{tabular}
& \begin{tabular}[c]{@{}c@{}}ESS\end{tabular} \\
\midrule
2014 & $182{\to}107$ & $9\to0$  & $2.059{\to}0.069$ & $0.950$ \\
2015 & $182{\to}156$ & $30\to0$ & $2.400{\to}0.085$ & $0.954$ \\
2016 & $194{\to}108$ & $1\to0$  & $2.104{\to}0.102$ & $0.953$ \\
2017 & $182{\to}202$ & $27\to0$ & $2.675{\to}0.144$ & $0.949$ \\
2018 & $182{\to}144$ & $18\to0$ & $2.652{\to}0.109$ & $0.949$ \\
\bottomrule
\end{tabular}
\end{minipage}
\endgroup
\vspace{-4pt}
\end{table}

\begin{table}[!t]
\caption{Tilted Mondrian for coverage disparity mitigation. %
}
\vspace{-4pt}
\label{tab:score-intervention-simple}
\centering
\small
\setlength{\tabcolsep}{2.0pt}
\renewcommand{\arraystretch}{1.05}
\begin{tabular*}{\textwidth}{@{\extracolsep{\fill}}llcrlcc@{}}
\toprule
Experiment & Score & \(L^0_{\max}\to L^\theta_{\max}\) & \(\max_z m_z^\theta\) & Harm & \(\Delta |C|\) \\
\midrule
Bias in Bios & SIMPLE & \(0.161\to0.090\) & \(0.071\) (surgeon) & \(0.010\) & \(0.012\) \\
Folktables occupation & SIMPLE & \(0.037\to0.014\) & \(0.024\) (Less than high school) & \(0.010\) & \(0.007\) \\
Folktables income & SIMPLE & \(0.033\to0.024\) & \(0.009\) (Bachelor) & \(0.003\) & \(0.001\) \\
ATUS & SIMPLE & \(0.010\to0.013\) & \(0.001\) (employed) & \(0.009\) & \(0.000\) \\
\bottomrule
\end{tabular*}
\vspace{-4pt}
\end{table}
\begin{table}[!t]
\caption{Tilted Mondrian for set size disparity mitigation.
}
\vspace{-4pt}
\label{tab:set-size-intervention-simple}
\centering
\small
\setlength{\tabcolsep}{2.0pt}
\renewcommand{\arraystretch}{1.05}
\begin{tabular*}{\textwidth}{@{\extracolsep{\fill}}llcllc@{}}
\toprule
Experiment & Score
& \(L^0_{\max,S}\to L^\theta_{\max,S}\)
& \(\max_z m_{z,S}^\theta\)
& Harm
& \(\Delta\mathrm{Cov} (\mathrm{pp})\) \\
\midrule
Bias in Bios
& SIMPLE
& \(0.065\to0.063\)
& \(0.005\) (psychologist)
& \(0.000\)
& \(-0.041\) \\

Folktables occupation
& SIMPLE
& \(0.180\to0.142\)
& \(0.058\) (High school or GED)
& \(0.020\)
& \(0.012\) \\

Folktables income
& SIMPLE
& \(0.188\to0.146\)
& \(0.041\) (Graduate or professional)
& \(0.003\)
& \(0.007\) \\

ATUS
& SIMPLE
& \(0.240\to0.238\)
& \(0.002\) (not in labor force)
& \(0.002\)
& \(0.000\) \\
\bottomrule
\end{tabular*}
\vspace{-4pt}
\end{table}
\vspace{-3pt}
\subsection{Evidence for Tilted Mondrian Disparity Mitigation}
\vspace{-3pt}

Table~\ref{tab:score-intervention-simple} reports the results for the SIMPLE score function at \(\alpha=0.1\). Tilted Mondrian reduces the test set's worst stratum-level gap in Bias in Bios and both Folktables tasks (year 2018) with only small average set-size changes; ATUS is an exception where very little disparity existed initially. %
Table~\ref{tab:appendix-score-intervention-remaining} in Appendix~\ref{app:no-free-score-intervention} reports the remaining score functions, where again Tilted Mondrian reduces disparities on the three datasets that initially display significant differences. These reductions are consistent with Theorem~\ref{prop:local-mitigation}, which shows that Tilted Mondrian can reduce finer-stratum coverage disparities. Table~\ref{tab:intervention-resplit-robustness} in Appendix~\ref{app:split-robustness} shows variability across data splits.

Table~\ref{tab:set-size-intervention-simple} demonstrates that Tilted Mondrian decreases the worst set size disparity for all four datasets. On Folktables we see relatively large reductions in the worst set size disparity. These reductions empirically support the set-size extension of Tilted Mondrian, showing that the algorithm can mitigate finer-stratum set-size disparities. %
Table~\ref{tab:set-size-repair-resplit-robustness} in Appendix~\ref{app:alternative-repair-target} shows variability across data splits.

\vspace{-4pt}
\section{Discussion and conclusion}\label{sec:limitations-conclusion}
\vspace{-4pt}

\noindent\textbf{Conclusion.} This paper studies a coarse-to-fine fairness problem in group-Mondrian CP. Aggregate group disparities can be shaped by population composition. Coarse group coverage guarantees may still conceal large differences within finer strata. Propensity standardization evaluates the composition effect by reweighting the set sizes over a standardized distribution. The resulting change in aggregate--stratum misalignment between the observed and standardized population reveals how strong the composition effect is on the disagreement. More broadly, this assessment captures how population mixture interacts with the disparity structure, information that cannot be easily obtained from base rates alone. Moving from assessment to mitigation, Tilted Mondrian intervenes directly at the score level without requiring model retraining. It provides a generic, modular refinement layer on top of a chosen Mondrian calibration scheme, mitigating fine-grained disparities while preserving groupwise coverage. The same principle applies to finer Mondrian hierarchies. Importantly, Tilted Mondrian does not require dividing the calibration data into increasingly small cells, offering the practitioner flexibility to choose a suitable grouping level of interest and still direct the intervention toward disparities at a finer level. A method-agnostic bound further quantifies the limits of such fine-grained fairness improvement. The Tilted Mondrian framework also extends to set size, mitigating disparities across finer strata while retaining groupwise coverage.

\newpage

\bibliographystyle{iclr2027_conference}
\bibliography{references}

\clearpage
\appendix
\section*{Appendix Contents}

\vspace{-0.3em}

\noindent
\begin{minipage}[t]{0.18\textwidth}
    \vspace{0pt}
    \centering
    \resizebox{2.19cm}{!}{
        \begin{tikzpicture}[
  x=2.8cm,
  y=2.8cm,
  line join=miter,
  cell/.style={draw=ink,line width=0.42pt}
]

\definecolor{paper}{RGB}{244,246,246}
\definecolor{paperwarm}{RGB}{239,238,231}
\definecolor{papercool}{RGB}{228,232,233}
\definecolor{yellow}{RGB}{228,206,87}
\definecolor{blue}{RGB}{0,39,127}
\definecolor{red}{RGB}{213,85,48}
\definecolor{ink}{RGB}{15,15,15}

\path[cell,fill=papercool] (0.000000,1.000000) -- (-0.233831,0.766169) -- (-0.091856,0.594712) -- (0.252456,0.747544) -- cycle;
\path[cell,fill=paper] (0.440355,0.559645) -- (0.200520,0.250496) -- (0.334089,-0.042142) -- (0.749191,0.250809) -- cycle;
\path[cell,fill=paperwarm] (-0.615823,0.384177) -- (-0.764039,0.235961) -- (-0.558245,-0.116109) -- (-0.375354,0.224582) -- cycle;
\path[cell,fill=paper] (0.200520,0.250496) -- (-0.124457,-0.168403) -- (0.209556,-0.275309) -- (0.328454,-0.029796) -- cycle;
\path[cell,fill=yellow] (0.749191,0.250809) -- (0.539075,0.102524) -- (0.824600,-0.175400) -- (1.000000,0.000000) -- cycle;
\path[cell,fill=papercool] (0.539075,0.102524) -- (0.334089,-0.042142) -- (0.401491,-0.189813) -- (0.692442,-0.046761) -- cycle;

\path[cell,fill=paperwarm] (-0.375354,0.224582) -- (-0.565545,-0.129708) -- (-0.125869,-0.170223) -- (-0.008479,-0.018907) -- cycle;
\path[cell,fill=paper] (-0.565545,-0.129708) -- (-0.671996,-0.328004) -- (-0.433385,-0.566615) -- (-0.125869,-0.170223) -- cycle;
\path[cell,fill=paperwarm] (-0.764039,0.235961) -- (-1.000000,0.000000) -- (-0.671996,-0.328004) -- (-0.558245,-0.116109) -- cycle;
\path[cell,fill=red] (-0.252387,-0.333306) -- (-0.433385,-0.566615) -- (-0.197433,-0.802567) -- (-0.062857,-0.462426) -- cycle;
\path[cell,fill=papercool] (-0.062857,-0.462426) -- (-0.197433,-0.802567) -- (0.000000,-1.000000) -- (0.294274,-0.705726) -- cycle;
\path[cell,fill=yellow] (-0.233831,0.766169) -- (-0.436534,0.563466) -- (-0.129078,0.382777) -- (-0.079830,0.580188) -- cycle;

\path[cell,fill=paper] (0.252456,0.747544) -- (-0.091856,0.594712) -- (0.196965,0.245914) -- (0.440355,0.559645) -- cycle;
\path[cell,fill=blue] (-0.079830,0.580188) -- (-0.197920,0.106822) -- (-0.008479,-0.018907) -- (0.196965,0.245914) -- cycle;
\path[cell,fill=paperwarm] (0.328454,-0.029796) -- (0.073839,-0.555552) -- (0.294274,-0.705726) -- (0.529573,-0.470427) -- cycle;
\path[cell,fill=ink] (-0.124457,-0.168403) -- (-0.252387,-0.333306) -- (0.073839,-0.555552) -- (0.209556,-0.275309) -- cycle;
\path[cell,fill=paper] (0.692442,-0.046761) -- (0.401491,-0.189813) -- (0.529573,-0.470427) -- (0.824600,-0.175400) -- cycle;
\path[cell,fill=papercool] (-0.436534,0.563466) -- (-0.615823,0.384177) -- (-0.197920,0.106822) -- (-0.129078,0.382777) -- cycle;

\draw[ink,line width=0.72pt]
  (0.000000,1.000000) -- (-1.000000,0.000000) --
  (0.000000,-1.000000) -- (1.000000,0.000000) -- cycle;
\end{tikzpicture}

    }
\end{minipage}%
\hfill
\begin{minipage}[t]{0.82\textwidth}
    \vspace{0pt}
    \begin{itemize}
        \item \apptoclink{app:proofs}
        {Appendix A: Proofs and technical discussions}

        \item \apptoclink{app:uncertainty-principle}
        {Appendix B: A product form set--error uncertainty principle}

        \item \apptoclink{app:experiments}
        {Appendix C: Experiment specific protocols and evaluation details}

        \item \apptoclink{app:targets}
        {Appendix D: Standardization targets and profile choice}

        \item \apptoclink{app:no-free-score-intervention}
        {Appendix E: Tilted Mondrian details}

        \item \apptoclink{app:split-robustness}
        {Appendix F: Finite sample robustness}

        \item \apptoclink{app:alternative-repair-target}
        {Appendix G: Tilted Mondrian for set size and fundamental tradeoffs}
    \end{itemize}
\end{minipage}

\section{Proofs and Technical Discussions}\label{app:proofs}

\subsection{Proofs of the main results}
This appendix records the proofs and technical discussions behind the curve summaries.  Throughout the appendix the predictor is fixed. The random variable being re-averaged is the delivered prediction-set size $M_\alpha(X)=|\mathcal{C}_\alpha(X)|$.

\begin{table}[h]
\caption{Notation often used in the main text and appendices. Tildes denote quantities induced by the profile target distribution $Q_H$; they do not change the delivered prediction sets.}
\label{tab:notation}
\centering
\footnotesize
\setlength{\tabcolsep}{3pt}
\renewcommand{\arraystretch}{1.08}
\begin{tabularx}{\columnwidth}{@{}p{0.34\columnwidth}X@{}}
\toprule
Symbol & Meaning \\
\midrule
$X,Y,\mathcal Y,K$ & Input, label, label space, and number of classes \\
$s(x,y),\hat q_\alpha$ & Nonconformity score and split-conformal threshold \\
$\mathcal C_\alpha(x),M_\alpha(X)$ & Delivered prediction set and set size \\
$\boldsymbol{\alpha}$ & Grid of nominal miscoverage levels used for curve reporting \\
$\alpha^\star$ & Selected information efficiency anchor\\
$M_\alpha^+(X)$ & Nonempty set-size version used in Proposition~\ref{prop:list-fano} \\
$\mathcal{H},\mathcal{H}_{\mathrm{bin}}$ & Entropy and binary entropy \\
$G\in\{a,b\},Z,\mathcal{Z},z$
& Comparison group, stratum variable, space, and a realized stratum \\
$\pi_g(z),\tilde\pi_g(z)$
& Raw and $Q_H$-induced stratum distributions \\
$V_g,V_{g,z},\tilde V_g,\tilde V_{g,z}$
& Raw and standardized group and stratum curves \\
$\Delta,\Delta_z,\tilde\Delta,\tilde\Delta_z$
& Raw and standardized aggregate and stratum gaps \\
$\mathbb E_Z,\mathbb E_{\tilde Z}$
& Expectations under the symmetric raw and standardized distributions \\
$\operatorname{Var}_Z,\operatorname{Var}_{\tilde Z}$
& Variance under the symmetric raw and standardized distributions \\
$R,\tilde R,\Gamma_z$
& Raw residual, standardized residual, and conflict strength \\
$H,P_H,e(h),Q_H$
& Profile, marginal distribution, propensity, and standardized distribution \\
$w_a(H),w_b(H)$
& Standardization factors $1-e(H)$ and $e(H)$ \\
$\mathcal A_{\rm raw},\mathcal A_{\rm std}$
& Raw and standardized alignment losses \\
$\eta_H$
& Average conditional reporting-stratum imbalance under $Q_H$ \\
$\Delta_{h,z}(\alpha)$
& Conditional group set-size gap at profile $h$ and stratum $z$ \\
$g^+(z),g^-(z)$
& Groups with higher and lower baseline coverage in stratum $z$ \\
$\theta_z^+,\theta_z^-$
& Nonnegative shifts for the higher/lower coverage groups in stratum $z$ \\
$S_\theta,\mathcal C_{\alpha,\theta}$
& Score after intervention and resulting set; $\mathcal C_{\alpha,0}$ is the baseline set \\
$c^0_{g,z},c^\theta_{g,z}$
& Baseline and intervened cell coverage \\
$d_z,d_z^\theta,m_z^\theta$
& Baseline stratum coverage gap, intervened gap, and gap reduction \\
$\epsilon$
& Harm budget for initially advantaged cells \\
Data splits
& $\mathcal{D}_{\mathrm{train}}$, $\mathcal{D}_{\mathrm{cal}}$, $\mathcal{D}_{\mathrm{prop}}/\mathcal{D}_{\mathrm{tilt}}$, and $\mathcal{D}_{\mathrm{test}}$ \\
\bottomrule
\end{tabularx}
\end{table}
\paragraph{Proof of Proposition~\ref{prop:common-profile-balance}.}
\begingroup
\renewcommand{\label}[1]{}
\commonprofilebalanceprop*
\endgroup
\begin{proof}
\label{proof:common-profile-balance}
Let $e(h)=\mathbb{P}(G=a\mid H=h)$.  For group $a$,
\begin{equation}
\begin{aligned}
&\mathbb{E}\{f(H)\mathbf{1}(G=a)(1-e(H))\}\nobreak\\
&\qquad =\int f(h)e(h)\{1-e(h)\} d P_H,
\end{aligned}
\end{equation}
and for group $b$,
\begin{equation}
\begin{aligned}
&\mathbb{E}\{f(H)\mathbf{1}(G=b)e(H)\}\nobreak\\
&\qquad =\int f(h)\{1-e(h)\}e(h) d P_H.
\end{aligned}
\end{equation}
The normalizing constants are identical. Both normalized weighted distributions are therefore $Q_H$.
\end{proof}

An equivalent symmetric form of \cref{eq:curve-aggregation}, in Section~\ref{sec:standardization}, is obtained by adding and subtracting the same cellwise terms. The raw aggregate gap can be written as
\begin{equation}
\begin{aligned}
\Delta(\alpha)
&=\frac12\sum_z\{\pi_a(z)+\pi_b(z)\}\Delta_z(\alpha)\nobreak\\
&\quad+\frac12\sum_z\{\pi_a(z)-\pi_b(z)\}
\{V_{a,z}(\alpha)+V_{b,z}(\alpha)\}.
\end{aligned}
\end{equation}
The first term is $\mathbb{E}_Z\Delta_Z(\alpha)$ under the symmetric raw stratum distribution, and the second term is $R(\alpha)$. The same calculation after standardization gives the corresponding identity for the $Q_H$-induced stratum distributions and the cell curves defined around \cref{eq:q-induced-cell-law}:
\begin{equation}
\begin{aligned}
\tilde\Delta(\alpha)
&=\frac12\sum_z\{\tilde\pi_a(z)+\tilde\pi_b(z)\}\tilde\Delta_z(\alpha)\nobreak\\
&\quad+\frac12\sum_z\{\tilde\pi_a(z)-\tilde\pi_b(z)\}
\{\tilde V_{a,z}(\alpha)+\tilde V_{b,z}(\alpha)\},
\end{aligned}
\label{eq:standardized-residual}
\end{equation}
which is the displayed identity \(\tilde\Delta(\alpha)=\mathbb{E}_{\tilde Z}\tilde\Delta_Z(\alpha)+\tilde R(\alpha)\); \cref{eq:standardized-residual} records the residual explicitly. For the residual bound, write $r_z=\tilde\pi_a(z)-\tilde\pi_b(z)$ and $s_z=\tilde V_{a,z}(\alpha)+\tilde V_{b,z}(\alpha)$. Since $\sum_z r_z=0$ and $0\le s_z\le 2K$,
\begin{equation}
\begin{aligned}
|\tilde R(\alpha)|
&=\frac12\left|\sum_z r_z\{s_z-K\}\right|\nobreak\\
&\le \frac K2\sum_z|r_z|
=K\,\operatorname{TV}(\tilde\pi_a,\tilde\pi_b).
\end{aligned}
\label{eq:standardized-residual-tv-bound}
\end{equation}
This is the residual bound used in Section~\ref{sec:standardization}. The centering step sharpens the direct range bound and shows that the remaining standardized residual is controlled by the total-variation distance between the two \(Q_H\)-induced stratum distributions.

\paragraph{Proof of Theorem~\ref{thm:alignment-loss}.}
\begingroup
\renewcommand{\label}[1]{}
\alignmentlossthm*
\endgroup
\begin{proof}
\label{proof:alignment-loss}
Write $p_g(z\mid h)=\mathbb P(Z=z\mid G=g,H=h)$. By
\cref{eq:q-induced-cell-law}, both propensity reweighted groups use the same distribution $Q_H$, so
\[
\tilde\pi_g(z)=\int p_g(z\mid h) d Q_H.
\]
The triangle inequality under this common mixture gives

\begin{align}
\operatorname{TV}(\tilde\pi_a,\tilde\pi_b)
&=\frac12\sum_z
\left|\int\{p_a(z\mid h)-p_b(z\mid h)\}d Q_H\right|\nonumber\\
&\le
\int\frac12\sum_z|p_a(z\mid h)-p_b(z\mid h)| d Q_H
=\eta_H.
\end{align}
Combining this with \cref{eq:standardized-residual-tv-bound} yields
\[
|\tilde R(\alpha)|
\le K\,\operatorname{TV}(\tilde\pi_a,\tilde\pi_b)
\le K\eta_H.
\]

Now suppose $\eta_H=0$. Since the integrand defining $\eta_H$ is
nonnegative,
\[
p_a(\cdot\mid h)=p_b(\cdot\mid h)
\]
for $Q_H$-almost every $h$; write the common conditional distribution as
$p(\cdot\mid h)$. Hence
\[
\tilde\pi_a(z)=\tilde\pi_b(z)
=\int p(z\mid h) d Q_H=:\tilde\pi(z),
\]
so \cref{eq:standardized-residual} immediately gives
$\tilde R(\alpha)=0$.

For brevity within the proof, let
\[
m_g(h,z)
:=
\mathbb E\{M_\alpha(X)\mid G=g,H=h,Z=z\},
\]
so that $\Delta_{h,z}(\alpha)=m_a(h,z)-m_b(h,z)$. 
From the definition of $\tilde V_{g,z}(\alpha)$ and the common conditional cell distribution,
\begin{align}
\tilde\Delta_z(\alpha)
&=\tilde V_{a,z}(\alpha)-\tilde V_{b,z}(\alpha)\nonumber\\
&=
\frac{\int p(z\mid h)\Delta_{h,z}(\alpha) d Q_H}
     {\tilde\pi(z)}\nonumber\\
&=
\mathbb E\{\Delta_{H,Z}(\alpha)\mid Z=z\},
\label{eq:standardized-cell-gap-conditional}
\end{align}
where the expectations and variances are under the joint distribution obtained by first drawing
$H\sim Q_H$ and then $Z\sim p(\cdot\mid H)$. Therefore, using
\cref{eq:alignment-loss-budget} and $\tilde R(\alpha)=0$,
\begin{align}
\mathcal A_{\rm std}(\alpha)
&=\operatorname{Var}_{\tilde Z}\{\tilde\Delta_Z(\alpha)\}\nonumber\\
&=\operatorname{Var}\!\left[
\mathbb E\{\Delta_{H,Z}(\alpha)\mid Z\}
\right].
\end{align}
The law of total variance now gives
\begin{align}
\operatorname{Var}
\{\Delta_{H,Z}(\alpha)\}
&=
\operatorname{Var}\!\left[
\mathbb E\{\Delta_{H,Z}(\alpha)\mid Z\}
\right]\nonumber\\
&\quad+
\mathbb E\!\left[
\operatorname{Var}\{\Delta_{H,Z}(\alpha)\mid Z\}
\right]\nonumber\\
&\ge \mathcal A_{\rm std}(\alpha),
\end{align}
which proves \(\mathcal A_{\rm std}(\alpha)\le\operatorname{Var}\{\Delta_{H,Z}(\alpha)\}\).

Now suppose composition dominance holds. By \cref{eq:composition-dominance},
\begin{align}
\operatorname{Var}
\{\Delta_{H,Z}(\alpha)\}
&=
\mathbb E
\!\left[
\left\{
\Delta_{H,Z}(\alpha)-
\mathbb E[\Delta_{H,Z}(\alpha)]
\right\}^2
\right]\nonumber\\
&\le
(
\operatorname*{ess\,sup}_{h,z}
\left|
\Delta_{h,z}(\alpha)-\mathbb E[\Delta_{H,Z}(\alpha)]
\right|
)^2\nonumber\\
&<
R^2(\alpha).
\label{eq:composition-dominance-variance}
\end{align}
By the definition
$R(\alpha)=\Delta(\alpha)-\mathbb E_Z\Delta_Z(\alpha)$,
\begin{align}
\mathcal A_{\rm raw}(\alpha)
&=
\mathbb E_Z
\left[
\left\{
R(\alpha)+
\mathbb E_Z\Delta_Z(\alpha)-
\Delta_Z(\alpha)
\right\}^2
\right]\nonumber\\
&=
R^2(\alpha)+
\operatorname{Var}_Z
\{\Delta_Z(\alpha)\}\nonumber\\
&\ge
R^2(\alpha).
\label{eq:raw-alignment-lower}
\end{align}
Combining
\cref{eq:composition-dominance-variance,eq:raw-alignment-lower}
gives
\begin{equation}
\mathcal A_{\rm std}(\alpha) \le
\operatorname{Var}
\{\Delta_{H,Z}(\alpha)\}<
R^2(\alpha) \le
\mathcal A_{\rm raw}(\alpha),
\end{equation}
and hence
\begin{equation}
\mathcal A_{\rm std}(\alpha)<\mathcal A_{\rm raw}(\alpha).
\end{equation}
If \(\Delta_{h,z}(\alpha)\) is constant for every \((h,z)\) pair, then \(\operatorname{Var}\{\Delta_{H,Z}(\alpha)\}=0\) and \(\mathcal{A}_{\rm std}(\alpha)=0.\)
\end{proof}

\paragraph{Proof of Theorem~\ref{prop:score-fix-validity}.}
\begingroup
\renewcommand{\label}[1]{}
\scorefixvaliditythm*
\endgroup
\begin{proof}
Condition on the tilt split. Then \(\theta\) is fixed before calibration, so \(S_\theta\) is an ordinary nonconformity score. Within each group \(g\), the calibration true label scores and the future true label score are exchangeable conditional on \(G=g\). The standard group Mondrian split conformal rank argument gives \(\mathbb P\{S_\theta(X,Y)\le \tilde{q}_g\mid G=g\}\ge 1-\alpha\), where \(\tilde{q}_g\) is the group specific conformal quantile computed on the calibration split after score shifting. Since \(Y\in \mathcal C_{\alpha,\theta}(X)\) exactly when \(S_\theta(X,Y)\le \tilde{q}_G\), the claim follows.
\end{proof}

\paragraph{Proof of Theorem~\ref{prop:local-mitigation}.}
\begingroup
\renewcommand{\label}[1]{}
\localmitigationprop*
\endgroup
\begin{proof}
Let $F_{g,z}$ be the baseline true-label score CDF, $q_g^0$ the baseline group
threshold, and write $\pi_i=\pi_g(z_i)$ and $\pi_j=\pi_g(z_j)$. Since both cells
are supported, $\pi_i,\pi_j>0$. Choose $\eta>0$ small enough that
\[
\delta_i=\eta\pi_j<d_{z_i},\qquad
\delta_j=\eta\pi_i<\min\{d_{z_j},\epsilon\}.
\]
Local positivity gives $\theta_i,\theta_j>0$ such that lowering scores in
$(g,z_i)$ by $\theta_i$ increases its coverage by $\delta_i$, while raising
scores in $(g,z_j)$ by $\theta_j$ decreases its coverage by $\delta_j$; set all
other shifts to zero. Since
\[
\pi_i\delta_i=\pi_j\delta_j,
\]
group coverage is unchanged. Hence
$d_{z_i}^\theta=d_{z_i}-\delta_i$ and
$d_{z_j}^\theta=d_{z_j}-\delta_j$, with all other gaps unchanged. The chosen
bounds prevent gap reversal and satisfy the harm constraint, proving
feasibility.
\end{proof}

\paragraph{Proof of Theorem~\ref{thm:no-free-repair-budget}.}
\begingroup
\renewcommand{\label}[1]{}
\nofreerepairbudgetthm*
\endgroup
\begin{proof}
Recall that \(\pi_g(z)=\mathbb P(Z=z\mid G=g)\), and that
\(\mathcal Z\) partitions each group population. For each stratum \(z\), let
\[
    h_z=c^0_{g^+(z),z}-c^\theta_{g^+(z),z},
    \qquad
    u_z=c^\theta_{g^-(z),z}-c^0_{g^-(z),z}.
\]
Thus \(h_z\) is the coverage loss on the initially advantaged side and \(u_z\) is the coverage gain on the initially disadvantaged side. Since \(d_z^\theta\ge0\), the single-cell mitigation satisfies \(m_z^\theta=d_z-d_z^\theta\le d_z\). The no-loss condition on disadvantaged cells gives \(u_z\ge0\). The no-gap-increase condition gives \(0\le h_z+u_z\le 2d_z\), because the intervened signed gap is \(d_z-h_z-u_z\). Hence
\[
    m_z^\theta=d_z-|d_z-h_z-u_z|\le h_z+u_z.
\]
The harm constraint gives \(h_z\le\epsilon\). Fixed group coverage gives, for every group \(g\),
\[
    \sum_{r:g^-(r)=g}\pi_g(r)u_r
    =
    \sum_{r:g^+(r)=g}\pi_g(r)h_r
    \le
    \epsilon\sum_{r:g^+(r)=g}\pi_g(r) .
\]
Because all disadvantaged-cell gains are nonnegative, every stratum \(z\) satisfies
\[
    u_z\le
    \frac{\epsilon\sum_{r:g^+(r)=g^-(z)}\pi_{g^-(z)}(r)}{\pi_{g^-(z)}(z)} .
\]
Combining this with \(h_z\le\epsilon\) gives
\[
    m_z^\theta\le h_z+u_z
    \le
    \epsilon\left(
    1+
    \frac{\sum_{r:g^+(r)=g^-(z)}\pi_{g^-(z)}(r)}{\pi_{g^-(z)}(z)}
    \right).
\]
Together with \(m_z^\theta\le d_z\), this yields the displayed cellwise bound. Taking the maximum over strata proves the result.
\end{proof}

\subsection{Finite-sample reporting stability and technical discussions}\label{app:finite-sample-reporting-stability}
\paragraph{Limitations.} Standardization requires adequate overlap and a profile that reflects the relevant composition; weak support, omitted information, task heterogeneity, or a weak classifier can result in lingering disagreement. The current framework is stated for binary group contrasts, and a multigroup formulation remains future work. Our Tilted Mondrian framework is interpretable in a fairness context but relies on a finite heuristic approach. The tilt split also trades the calibration sample size for score adjustment selection and calibration. Future work may focus on more efficient constrained optimization for establishing fine-grained fairness with explicit cell-level costs and trade-offs. Evaluation across additional real-world settings would further assess the framework's generality.

As noted in Section~\ref{sec:empirical-protocol-evidence}, the normalized Effective Sample Size (ESS) serves as a diagnostic for our standardization. The behavior of the ESS exposes a limit of identifiability. If profile features strongly separate the demographic groups, overlap weights vanish for the majority of samples, causing the ESS to drop. This warns practitioners that the standardized gap is driven by a thin subpopulation. In the extreme case where profile features perfectly proxy group membership (i.e., $e(h) \in \{0, 1\})$, all overlap weights become zero, rendering the standardized disparity mathematically undefined.

\paragraph{Fixed-policy reporting.}
All reporting identities are applied after a thresholding policy has already been chosen and the prediction sets have been formed. The policy may be pooled or group-conditional. The raw and standardized curves then re-average the same delivered set sizes under different reporting populations.  Thus the algebra below does not require the same thresholding policy in every experiment; it requires that, within each experiment, the thresholding policy be fixed before reporting weights and curve summaries are computed.

\paragraph{Split-conformal rank argument.}
After the score rule is specified, the calibration scores $U_i=s(X_i,Y_i)$ and the test score $U_{n+1}=s(X,Y)$ are exchangeable. With the usual conservative split-conformal order statistic, the event $Y\notin\mathcal{C}_\alpha(X)$ is the event that $U_{n+1}$ exceeds the selected calibration threshold.  Exchangeability implies that this event has probability at most $\alpha$, up to the standard finite-sample conservativeness from ties and order-statistic rounding.  The groupwise statement applies the same rank argument after conditioning on $G=g$ and using only calibration examples from that group.

\paragraph{Finite-grid H\'ajek stability.}
For group $g$, the population standardized curve is
\begin{equation}
    \tilde V_g(\alpha)
    =
    \frac{\mathbb{E}\{\mathbf{1}(G=g)w_g(H)M_\alpha(X)\}}
         {\mathbb{E}\{\mathbf{1}(G=g)w_g(H)\}},
\end{equation}
where $w_a(H)=1-e(H)$ and $w_b(H)=e(H)$. On the test split, the corresponding stabilized H\'ajek report is the finite-sample ratio in \cref{eq:hajek-curve}, denoted by $\tilde V'_g(\alpha)$.  Cellwise reports use the same ratio after multiplying numerator and denominator by $\mathbf{1}(Z_i=z)$.  This normalized ratio form is the classical H\'ajek style estimator \citep{hajek1971comment}; the bounded weights used here follow the propensity reweighting construction of \citet{li2018balancing}.

Suppose the finite-sample propensity standardized factors are bounded, the corresponding population denominator is bounded away from zero, and the operating grid $\Lambda$ is finite.  Since $0\le M_\alpha(X_i)\le K$, concentration for bounded averages \citep{hoeffding1963probability} gives uniform convergence of the numerator and denominator over $\alpha\in\Lambda$.  The ratio map is Lipschitz on denominators bounded away from zero, so
\begin{equation}
    \sup_{\alpha\in\Lambda}
    \left|\tilde V'_g(\alpha)-\tilde V_g(\alpha)\right|
    =o_p(1)
\end{equation}
when the fitted weights converge to the population standardized weights and the stabilized denominator remains nondegenerate. The same argument applies to cellwise reports with positive weighted denominators. This is a stability statement for the reported curve estimates rather than a coverage argument.

\paragraph{Plug-in propensity estimates.}
The standardized target in the population is written in terms of $e(H)=\mathbb{P}(G=a\mid H)$, whereas the finite-sample reports use a fitted propensity $e'(H)$ learned on the propensity split and then held fixed on the test split. This is the standard plug-in use of estimated propensity scores in weighting estimators \citep{hirano2003efficient}. In our setting the claim is descriptive rather than causal: the fitted propensity defines the standardized report used to re-average already delivered prediction-set sizes. If the stabilized weight is a Lipschitz function of the fitted propensity on the clipped range and $\mathbb{E}|e'(H)-e(H)|=r_n$, then the population weighted curve formed with $e'$ differs from the ideal standardized curve by at most a constant multiple of $Kr_n$, provided the group denominators are bounded away from zero. This follows by writing the difference between two ratio functionals, bounding the numerator and denominator perturbations by the weight perturbation, and using the boundedness of $M_\alpha$. The fixed bounded-weight rule and H\'ajek normalization make this perturbation controlled in finite samples, while balance checks assess whether the learned weighted report is stable \citep{austin2015moving}.

\paragraph{Reporting cell interpretation.}
Propensity standardization aligns the groups in the coordinates of the chosen prediction profile $H$ \citep{rosenbaum1983central,li2018balancing}. Under the shared profile distribution, reporting-cell balance is governed by the within-profile distributions $\mathbb P(Z\in\cdot\mid G,H)$. Theorem~\ref{thm:alignment-loss} summarizes their remaining discrepancy through $\eta_H$ and its composition effect through $\tilde R$. A standardized cell curve applies the learned reweighting within each test stratum $Z=z$. This design keeps $H$ as the frozen standardization profile and $Z$ as the reporting coordinate for cellwise Simpson-type disagreement. Exact induced distribution equality follows under $Z\perp G\mid H$. We therefore report both profile balance and residual reporting-cell mixture summaries.

\paragraph{Post-conformal reporting note.}
The delivered set $\mathcal{C}_\alpha(x)$ is determined by the specified score and the split conformal calibration quantile. Under the separated reporting protocol in Section~\ref{sec:empirical-protocol-evidence}, the profile map, propensity fit, stabilization rule, and H\'ajek averages are applied after these sets have been fixed. The score, order statistic, and membership rule in \cref{eq:conformal-set} remain unchanged, so the event $Y\in\mathcal{C}_\alpha(X)$ is identical before and after reporting standardization.

\section{A Product Form Set--Error Uncertainty Principle}
\label{app:uncertainty-principle}

We first give the full Fano argument \citep{correia2024information} underlying
the intrinsic ambiguity scale presented in Section~\ref{sec:curves}.
Throughout, all logarithms are natural. Write \(\mathcal H(Y\mid X)\) for the
conditional entropy of \(Y\) given \(X\), let
\(\mathcal H_{\mathrm{bin}}(u)=-u\log u-(1-u)\log(1-u)\), and denote
\([t]_+=\max(t,0)\).
\begin{restatable}[Fano set-size floor]{proposition}{listfanosizeprop}
\label{prop:list-fano}
For the split CP sets in
\cref{eq:conformal-set}. %
Let \(\mathcal D_{\rm cal}=\{(X_i,Y_i)\}_{i=1}^n\) and the fresh test point \((X,Y)\) be i.i.d. draws from the same distribution with the nonempty set size convention $M_\alpha^+(X)=M_\alpha(X)\vee1$.
If $\alpha<1/2$, then
\begin{equation}
\label{eq:entropy-floor}
    \mathbb{E}M_\alpha^+(X)\ge
    \exp\Bigl\{\bigl[
    \mathcal{H}(Y\mid X)
    -\mathcal{H}_{\mathrm{bin}}(\alpha)
    -\alpha\log(K-1)
    \bigr]_+\Bigr\}.
\end{equation}
\end{restatable}
\begin{proof}
Let $\mathcal C_\alpha^+(X)$ be the nonempty version of the set and let
$E=\mathbf 1\{Y\notin\mathcal C_\alpha^+(X)\}$.
Because the nonempty replacement only adds a label when the original
set is empty, split-conformal validity gives
$\mathbb P(E=1)\le\alpha$. Since \(D_{\mathrm{cal}}\perp(X,Y)\) and \(E\) is determined by
\((X,Y,D_{\mathrm{cal}})\),
\begin{equation}
\begin{aligned}
\mathcal H(Y\mid X)
&=\mathcal H(Y\mid X,D_{\mathrm cal})\\
&=\mathcal H(E\mid X,D_{\mathrm cal})
  +\mathcal H(Y\mid X,D_{\mathrm cal},E)\\
&\le
\mathcal H_{\rm bin}\!\bigl(\mathbb P(E=1)\bigr)
+\mathbb E\!\left[(1-E)\log M_\alpha^+(X)\right]
+\mathbb P(E=1)\log(K-1)\\
&\le
\mathcal H_{\rm bin}\!\bigl(\mathbb P(E=1)\bigr)
+\mathbb E\log M_\alpha^+(X)
+\mathbb P(E=1)\log(K-1).
\end{aligned}
\end{equation}
Since $\mathbb P(E=1)\le\alpha<1/2$,
\begin{equation}
\mathcal H(Y\mid X)
\le
\mathcal H_{\mathrm{bin}}(\alpha)
+\mathbb E\log M_\alpha^+(X)
+\alpha\log(K-1).
\end{equation}
Because $M_\alpha^+(X)\ge1$, the positive part may be taken.
Jensen's inequality then gives
$\mathbb E M_\alpha^+(X)
\ge\exp\{\mathbb E\log M_\alpha^+(X)\}$,
which yields \cref{eq:entropy-floor}.
\end{proof}

\noindent\textbf{Take-away.} 
CP sets cannot be arbitrarily small. When the label \(Y\) remains hard to predict from the input \(X\), the procedure must use larger sets on average.

\paragraph{Price of increasing coverage.}
Write
\[
C:=\exp\{\mathcal H(Y\mid X)\}
\]
and define
\begin{equation}
L_K (\alpha)
:=
C\exp\!\left\{
\alpha\log(\alpha)
+(1-\alpha)\log(1-\alpha)
-\alpha\log(K-1)
\right\}.
\label{eq:LK-floor}
\end{equation}
For $\alpha<1/2$, Proposition~\ref{prop:list-fano} gives
$\mathbb E[M_\alpha^+(X)]\ge L_K (\alpha)$ whenever this is the
nontrivial branch of the floor. Substituting nominal coverage
$\gamma=1-\alpha\in(1/2,1)$, define
\begin{equation}
G_K (\gamma):=L_K (1-\gamma).
\label{eq:GK-coverage}
\end{equation}
Thus $G_K (\gamma)$ is the information-theoretic lower bound at
nominal coverage $\gamma$. For $K\ge2$ and $1/2<\gamma<1$,
\begin{equation}
G_K '(\gamma)
=
G_K (\gamma)
\left\{
\log\!\left(\frac{\gamma}{1-\gamma}\right)
+\log(K-1)
\right\}
>0,
\label{eq:GK-first}
\end{equation}
and
\begin{equation}
G_K ''(\gamma)
=
G_K (\gamma)
\left[
\left\{
\log\!\left(\frac{(K-1)\gamma}{1-\gamma}\right)
\right\}^{2}
+
\frac{1}{\gamma(1-\gamma)}
\right]
>0.
\label{eq:GK-second}
\end{equation}
Therefore, $G_K (\gamma)$ is strictly increasing and strictly convex in
nominal coverage. Increasing nominal coverage raises the lower bound
set size burden, and this burden rises progressively faster as coverage
approaches one.

The first derivative measures the instantaneous increase in the lower bound
per unit increase in nominal coverage. Define the proportional price
\begin{equation}
r_K (\gamma)
:=
\frac{G_K '(\gamma)}{G_K (\gamma)}
=
\log\!\left(\frac{(K-1)\gamma}{1-\gamma}\right).
\label{eq:rK-price}
\end{equation}
For a small coverage increase $d\gamma>0$,
\begin{equation}
\frac{G_K (\gamma+d\gamma)-G_K (\gamma)}
{G_K (\gamma)}
\approx
r_K (\gamma)d\gamma.
\label{eq:rK-local}
\end{equation}
Hence a one percentage point coverage increase, $d\gamma=0.01$, corresponds
approximately to a $r_K (\gamma)\%$ increase in the lower-bound burden.
Moreover,
\begin{equation}
r_K '(\gamma)
=
\frac{1}{\gamma(1-\gamma)}
>0,
\qquad
\lim_{\gamma\to1}r_K (\gamma)=\infty.
\label{eq:rK-divergence}
\end{equation}
Thus the proportional lower-bound price becomes arbitrarily large as
nominal coverage approaches one. For example, the same one-percentage-point
increase in coverage is cheaper (in terms of the set-size lower bound) when moving from $80\%$ to
$81\%$ than when moving from $95\%$ to $96\%$.

\paragraph{Finite coverage increments.}
Let $0<\delta<\alpha<1/2$. Define the relative lower-bound inflation
associated with increasing nominal coverage by $\delta$ as
\begin{equation}
R_{K, \delta}(\alpha)
=
\frac{L_K (\alpha-\delta)-L_K (\alpha)}
{L_K (\alpha)}.
\label{eq:RK-finite}
\end{equation}
For fixed $K$ and $\delta$, $R_{K, \delta}(\alpha)$ is strictly decreasing
in $\alpha$. Suppose a practitioner wants to increase nominal coverage by
$\delta$, equivalently decreasing the error rate from $\alpha$ to
$\alpha-\delta$, and is willing to tolerate at most a proportional
lower-bound inflation $r>0$. The change is within tolerance exactly when
\begin{equation}
\begin{aligned}
R_{K, \delta}(\alpha)\le r
&\iff
\frac{L_K (\alpha-\delta)}{L_K (\alpha)}
\le r+1\\
&\iff
\log L_K (\alpha-\delta)-\log L_K (\alpha)
\le \log(r+1).
\end{aligned}
\label{eq:RK-tolerance}
\end{equation}
Because $R_{K, \delta}(\alpha)$ is strictly decreasing in $\alpha$, there is
at most one interior cutoff $ \alpha^*$ satisfying
\begin{equation}
\log L_K ( \alpha^*-\delta)-\log L_K ( \alpha^*)=\log(r+1).
\label{eq:RK-cutoff}
\end{equation}
Whenever such a cutoff exists,
\begin{equation}
\alpha\ge \alpha^*
\ \Longrightarrow\
R_{K, \delta}(\alpha)\le r,
\qquad
\alpha< \alpha^*
\ \Longrightarrow\
R_{K, \delta}(\alpha)>r.
\label{eq:RK-regimes}
\end{equation}
Now, the question is when does the cutoff exist? For a fixed $\delta$, the possible current error rates satisfy $\delta < \alpha < 1/2$. Since $R_{K, \delta}(\alpha)$ is strictly decreasing in $\alpha$, its smallest possible value is 
\begin{equation*}
    R_{\min} = \lim_{\alpha \to 1/2} R_{K, \delta}(\alpha) = \frac{L_K(1/2 - \delta)}{L_K(1/2)} - 1,
\end{equation*}
and its largest possible value is
\begin{equation*}
    R_{\max} = \lim_{\alpha \to \delta} R_{K, \delta}(\alpha) = \frac{L_K(0)}{L_K(\delta)} - 1 = \exp\{\mathcal{H}_{\text{bin}}(\delta) + \delta \log(K-1)\} - 1.
\end{equation*}
Therefore, 
\begin{itemize}
    \item If $r \le R_{\min}$, the desired coverage gain is never affordable within $0 < \alpha < 1/2$.
    \item If $r \ge R_{\max}$, it is affordable for every admissible starting $\alpha$.
    \item If $R_{\min} < r < R_{\max}$, there is a unique cutoff $\alpha^*$. This cutoff can be calculated by one-dimensional root finding for \cref{eq:RK-cutoff}.
\end{itemize}

The above computation illustrates an important operating rule. It provides a useful guide on coverage that one hopes to attain and inflation in set size that one expects to tolerate. For a proposed coverage increase $\delta$ and a tolerated proportional inflation $r$ in the set-size lower bound, one can solve \cref{eq:RK-cutoff} once and compare the current error level $\alpha$ with $\alpha^*$. If $\alpha\ge\alpha^*$, the proposed coverage increase remains within the tolerance; if $\alpha<\alpha^*$, the same coverage increase violates that tolerance, indicating that one must either request a smaller coverage increment or accept a larger set-size inflation.

The expected set size curve records the arithmetic burden delivered to a user. Multiplying the floor in \cref{eq:entropy-floor} by the corresponding error channel factor gives %
\begin{equation}
\underbrace{
\mathbb E M_\alpha^+(X)
}_{\substack{\text{expected set-size}\\\text{burden}}}
\times
\underbrace{
\exp\!\left\{\mathcal H_{\mathrm{bin}}(\alpha)\right\}
(K-1)^\alpha
}_{\substack{\text{effective error-channel}\\\text{uncertainty}}}
\geq
\underbrace{
\exp\!\left\{\mathcal H(Y\mid X)\right\}
}_{\text{intrinsic label perplexity}}.
\label{eq:product-set-error-uncertainty}
\end{equation}
The two factors represent uncertainty retained for downstream selection and uncertainty delegated to tolerated misses. Thus they cannot be simultaneously compressed below intrinsic label perplexity. In this product form sense, \Cref{eq:product-set-error-uncertainty} has a trade-off interpretation. That is, reducing the burden carried by one channel requires the other to represent the remainder.

\begin{restatable}[Information efficiency anchor]{proposition}
{informationanchorprop}
\label{prop:uncertainty-anchor}
For the split CP family in \cref{eq:conformal-quantile,eq:conformal-set}, suppose that \(\alpha<1/2\) for every \(\alpha \in \boldsymbol{\alpha}\), and define
\begin{equation}
\alpha^\star\in\arg\min_{\alpha\in\boldsymbol{\alpha}}
\left\{
\mathbb E M_\alpha^+(X)
\exp\!\left\{\mathcal H_{\mathrm{bin}}(\alpha)\right\}
(K-1)^\alpha
\right\}.
\label{eq:uncertainty-anchor}
\end{equation}
Then \(\alpha^\star\) is the valid grid level with the smallest multiplicative excess above intrinsic label perplexity, and hence comes closest to saturating \Cref{eq:product-set-error-uncertainty}.
\end{restatable}
\begin{proof}
The entropy floor in \Cref{eq:product-set-error-uncertainty} is constant across the finite grid \(\boldsymbol{\alpha}\). Hence minimizing the product in \Cref{eq:uncertainty-anchor} also minimizes its multiplicative excess above that floor.
\end{proof}

\noindent\textbf{Take-away.}
Selecting coverage levels trades smaller prediction sets against more tolerated errors. The proposition identifies the level that comes closest to the best possible balance between the two.

\paragraph{Operational implications.}
In a shortlist-and-fallback system, the set channel retains candidates for human or algorithmic selection, whereas tolerated misses are delegated to escalation, additional measurement, or exception handling. Therefore, the anchor \(\alpha^\star\) provides a reference for the least redundant allocation between these uncertainty channels.

The principle also identifies when set valued output is necessary. If the error channel is smaller than \(\exp\{\mathcal{H}(Y\mid X)\}\), then \Cref{eq:product-set-error-uncertainty} forces \(\mathbb P\{M_\alpha^+(X)>1\}>0\). That is, a singleton output cannot represent the intrinsic uncertainty. In addition, the minimized product can also compare procedures, with a smaller excess indicating less redundant uncertainty representation. These uses do not make \(\alpha^\star\) a universal deployment optimum, because cost of misses, review, and escalation remain application dependent. For a broader information theoretic perspective on conformal prediction, see \citet{correia2024information}.

\section{Experiment Protocols and Evaluation Details}\label{app:experiments}\label{app:protocols}

This appendix collects experimental details and evaluation summaries for the empirical studies in Sections~\ref{sec:empirical-protocol-evidence} and~\ref{sec:no-free-repair-intervention}. The experiments use Bias in Bios \citep{dearteaga2019bias}, Folktables/ACS \citep{ding2021retiring}, and ATUS \citep{bls2024atus}. In the standardization table, an arrow denotes raw to standardized re-averaging on the same test prediction sets. The score, conformal threshold, and prediction set membership rule are fixed before reporting; reporting weights enter only the curve averages on the test split. Our code is available at \url{https://github.com/GreenPenguin001/tilted-mondrian}.

For a supported grid point $(z,\alpha)$ on the test split, a reversal is recorded when the aggregate and stratum gaps have opposite signs. Its raw conflict strength is $\Gamma'_z(\alpha)=[-\Delta'(\alpha)\Delta'_z(\alpha)]_+$, and the standardized version uses the gaps analogously. Total conflict strength sums this quantity over the supported test grid. High conflict strength is a supported reversal with $\Gamma'_z(\alpha)>0.020$. This threshold is descriptive only; the continuous conflict strength remains the main magnitude summary.  Effective sample size is the Kish ESS of all test reporting weights divided by the total test sample size.  Balance is reported in the coordinates of the fixed profile $H$ and, when relevant, in the reporting stratum mixture.

Finite sample standardization estimates use stabilized H\'ajek averages. On $\mathcal{D}_{\mathrm{test}}$, let $e'_i=e'(H_i)$ be the group propensity fitted on $\mathcal{D}_{\mathrm{prop}}$ and then treated as fixed. The standardized factors are $r'_{a,i}=1-e'_i$ for group $a$ and $r'_{b,i}=e'_i$ for group $b$. After the prespecified bounded-weight stabilization, the within group weights are
\begin{equation}
    w'_{g,i}
    =
    \frac{r'_{g,i}}
         {\sum_{j\in\mathcal{D}_{\mathrm{test}}}\mathbf{1}(G_j=g)r'_{g,j}}.
\end{equation}
The aggregate standardized curve is estimated by
\begin{equation}
    \tilde V'_g(\alpha)=\sum_{i\in\mathcal{D}_{\mathrm{test}}}\mathbf{1}(G_i=g)w'_{g,i}M_\alpha(X_i).
    \label{eq:hajek-curve}
\end{equation}
Cellwise curves use the same standardized factors with the H\'ajek ratio restricted to $Z_i=z$.

\begin{algorithm}[!ht]
\caption{Propensity Standardization for Disparity Assessment (detailed version of Algorithm~\ref{alg:main-curve-reporting})}
\label{alg:raw-overlap-curve-reporting}

\noindent\textbf{Input:}
Disjoint splits
$\mathcal D_{\mathrm{train}},\mathcal D_{\mathrm{prop}},
\mathcal D_{\mathrm{cal}},\mathcal D_{\mathrm{test}}$;
group $G$, reporting stratum $Z$, profile rule $H$, score policy,
thresholding policy, operating grid $\boldsymbol{\alpha}$, and supported-cell rule.

\vspace{0.4ex}
\begin{algorithmic}[1]
\STATE Fit classifier $f$ on $\mathcal D_{\mathrm{train}}$; freeze $f$, the score policy, and the induced class probabilities.
\STATE Construct $H$ from the frozen outputs. On $\mathcal D_{\mathrm{prop}}$, fit $e'(h)=\widehat{\mathbb P}(G=a\mid H=h)$; freeze the propensity, bounded-weight stabilization, and supported-cell rules.
\STATE On $\mathcal D_{\mathrm{cal}}$, obtain the threshold(s) prescribed by the fixed thresholding policy for every $\alpha\in\boldsymbol{\alpha}$.
\STATE On $\mathcal D_{\mathrm{test}}$, compute $H_i$, $e'_i=e'(H_i)$, and standardized factors $r'_{a,i}=1-e'_i$, $r'_{b,i}=e'_i$; apply the frozen stabilization and H\'ajek-normalize within group to obtain $w'_{g,i}$.
\FOR{$\alpha\in\boldsymbol{\alpha}$}
    \STATE Construct each $\mathcal C_{\alpha,i}$ once under the fixed
thresholding policy; set
$M_{\alpha,i}=|\mathcal C_{\alpha,i}|$.
    \STATE Compute raw group and cell curves $V'_g(\alpha),V'_{g,z}(\alpha)$ by unweighted means on the frozen supported grid.
    \STATE Compute $\tilde V'_g(\alpha)$ by \cref{eq:hajek-curve} and $\tilde V'_{g,z}(\alpha)$ by the same H\'ajek ratio restricted to $Z_i=z$.
    \STATE Form raw and standardized aggregate/stratum gaps; record sign reversals and product strengths by \cref{eq:simpson-product-strength}.
\ENDFOR
\STATE Record ESS fractions and profile/cell-mixture balance for the frozen reporting weights.
\end{algorithmic}

\vspace{0.4ex}
\noindent\textbf{Output:}
Raw and standardized group/cell curves; reversal counts; conflict strengths; effective sample size; balance checks.

\vspace{0.3ex}
\noindent\textbf{Invariant:}
Both reports use the same test values $M_\alpha(X_i)$ and the same supported $(z,\alpha)$ grid. Reweighting changes neither $f$, the score, the calibrated threshold(s), nor $\mathcal C_\alpha(X_i)$.
\end{algorithm}

\begin{table*}[t]
\caption{Exact split sizes for the empirical reports. The split columns are training, calibration, propensity/tilt, and testing.}
\label{tab:appendix-split-sizes}
\centering
\small
\setlength{\tabcolsep}{4.8pt}
\renewcommand{\arraystretch}{1.08}
\begin{tabular*}{0.95\textwidth}{@{\extracolsep{\fill}}llrrrrr@{}}
\toprule
Setting & Year(s) & Train & Cal. & Prop./Tilt & Test & Total \\
\midrule
Bias in Bios & all & 127,023 & 79,390 & 79,390 & 31,756 & 317,559 \\
Folktables occupation & 2014 & 163,527 & 102,204 & 102,205 & 40,882 & 408,818 \\
Folktables occupation & 2015 & 162,130 & 101,331 & 101,332 & 40,533 & 405,326 \\
Folktables occupation & 2016--2018 each & 164,195 & 102,622 & 102,622 & 41,049 & 410,488 \\
Folktables income quintile & 2014--2018 each & 164,195 & 102,622 & 102,622 & 41,049 & 410,488 \\
ATUS & all & 71,356 & 44,683 & 44,655 & 17,783 & 178,477 \\
\bottomrule
\end{tabular*}
\end{table*}

\paragraph{Conformal score construction.}
Let $\hat p_y(x)$ denote the frozen fitted probability for candidate label
$y$, let
$\hat p_{(1)}(x)\ge\cdots\ge\hat p_{(K)}(x)$
be the sorted class probabilities, and let $r_y(x)$ be the rank of $y$.
For randomized adaptive scores, let $U\sim{\rm Unif}(0,1)$ denote the
auxiliary randomization, treated as part of the frozen score rule. All
scores follow the convention that smaller values are more conforming:
\begin{align}
s_{\rm SIMPLE}(x,y)
&=1-\hat p_y(x),\\
s_{\rm APS}(x,y;U)
&=\sum_{j<r_y(x)}\hat p_{(j)}(x)+U\hat p_y(x),\\
s_{\rm RAPS}(x,y;U)
&=s_{\rm APS}(x,y;U)
+\lambda\{r_y(x)-k_{\rm reg}\}_+,\\
s_{\rm SAPS}(x,y;U)
&=
\begin{cases}
U\hat p_{(1)}(x), & r_y(x)=1,\\
\hat p_{(1)}(x)+\lambda\{r_y(x)-2+U\}, & r_y(x)>1.
\end{cases}
\end{align}
The experiment-specific score parameters and thresholding policies are
listed in Table~\ref{tab:appendix-score-grids}. In every experiment, the
score rule is fixed before split conformal calibration, and reporting
weights enter neither the score nor the calibration quantile.

\begin{table*}[t]
\caption{Score policies, operating grids, and support rules. All raw-to-standardized arrows use the same test prediction sets.}
\label{tab:appendix-score-grids}
\centering
\footnotesize
\setlength{\tabcolsep}{3.5pt}
\renewcommand{\arraystretch}{1.12}
\begin{tabularx}{\textwidth}{@{}>{\raggedright\arraybackslash}p{0.18\textwidth}>{\raggedright\arraybackslash}p{0.19\textwidth}>{\raggedright\arraybackslash}p{0.22\textwidth}>{\raggedright\arraybackslash}p{0.17\textwidth}>{\raggedright\arraybackslash}X@{}}
\toprule
Setting & Threshold policy & Scores & $\alpha$ grid & Support rule \\
\midrule
Bias in Bios & Gender-conditional thresholds & SIMPLE; APS; SAPS($\lambda=0.2$); RAPS($\lambda=0.2$, $k_{\mathrm{reg}}=5$) & 0.020 to 0.100; step 0.005 & At least 200 test observations per Gender group within each occupation stratum; the same supported grid is used for raw and standardized summaries. \\
Folktables occupation & Pooled threshold within year & SIMPLE & 0.010 to 0.100; step 0.001 & At least 200 test observations per group within each education stratum. \\
Folktables income quintile & Pooled threshold within year & SIMPLE & 0.010 to 0.100; step 0.001 & Same test education stratum support rule as Folktables occupation. \\
ATUS & Pooled threshold & SIMPLE primary; APS, RAPS($\lambda=0.01$, $k_{\mathrm{reg}}=5$), and SAPS($\lambda=0.2$) checks & 0.020 to 0.250; step 0.010 & Prespecified employment stratum support rule. \\
\bottomrule
\end{tabularx}
\end{table*}

\paragraph{Prediction profile construction.}
For Bias in Bios and the two Folktables tasks, $H$ is built from the frozen class-probability vector. For a $K$-class task, the profile contains one-hot indicators for the top first, second, and third predicted labels, the three largest probabilities $p_{(1)},p_{(2)},p_{(3)}$, normalized entropy, the two margins $p_{(1)}-p_{(2)}$ and $p_{(2)}-p_{(3)}$, and the top-three mass.  Its dimension is therefore $3K+7$, giving 37 features for $K=10$ and 22 features for $K=5$. ATUS uses four probability features because the final closed-world task has the top-10 activity vocabulary but the implemented profile rule omits top-label indicators. Across all experiments, $H$ excludes the true group, true label, reporting stratum, realized set size, conformal thresholds, alpha-specific membership indicators, and final curve summaries.

\begin{table*}[t]
\caption{Prediction profiles and propensity models for the standardized report. The reporting weight model enters only the curve summaries on the test split.}
\label{tab:appendix-profile-overlap}
\centering
\footnotesize
\setlength{\tabcolsep}{3.0pt}
\renewcommand{\arraystretch}{1.14}
\begin{tabularx}{0.94\textwidth}{@{}>{\raggedright\arraybackslash}p{0.15\textwidth}>{\raggedright\arraybackslash}p{0.24\textwidth}>{\raggedright\arraybackslash}p{0.29\textwidth}>{\raggedright\arraybackslash}X@{}}
\toprule
Setting & Predictor and score & Profile $H$ & Propensity model \\
\midrule
Bias in Bios & DistilBERT top-10 occupation classifier; SIMPLE, APS, SAPS, and RAPS. & 37-dimensional frozen probability profile. & Feed-forward neural-network propensity model with hidden layers $(128,64)$ and ReLU activations; fitted on the propensity split. \\
Folktables occupation & Histogram gradient boosting on ACS person records, one model per year; SIMPLE. & 37-dimensional frozen softmax profile; Gender and education are excluded from $H$. & Logistic propensity model with balanced class weights; fitted separately by year on the propensity split. \\
Folktables income quintile & Histogram gradient boosting for income quintiles, one model per year; SIMPLE. & 22-dimensional frozen softmax profile; race and education are excluded from $H$. & Logistic propensity model with balanced class weights; fitted separately by year on the propensity split. \\
ATUS & Regularized histogram gradient boosting on diary, time, calendar, demographic, region, and household features; SIMPLE with APS, RAPS, and SAPS checks. & Four probability features: $p_{\mathrm{top1}}$, $p_{\mathrm{top2}}$, margin, and entropy. & Logistic propensity model fitted on the propensity split. \\
\bottomrule
\end{tabularx}
\vspace{0.5mm}
\begin{minipage}{0.94\textwidth}
\small
\emph{Notes.} In the logistic rows, standardized means that the coordinates of $H$ are centered and scaled before logistic fitting. All standardized reports use prespecified bounded weights and H\'ajek normalization. Exact training hyperparameters for all base and propensity models, including learning rates, regularization, batch sizes, epochs or iterations, and random seeds, are provided in the supplementary code.
\end{minipage}
\end{table*}

\paragraph{Empirical alignment loss decomposition.}
Table~\ref{tab:alignment-decomposition} reports the realized component changes. Squared residuals decrease by $59.7$--$99.3\%$ in every setting, and stratum-gap variance decreases by $55.0$--$92.6\%$ for Bias in Bios and Folktables. ATUS has increased stratum-gap variance, while its composition reduction yields a $36.1$--$57.0\%$ net alignment-loss decrease.

\begin{table}[H]
\caption{Relative reductions (\%) in grid-averaged alignment loss and its two components. Ranges are across scores or years; negative values denote increases.}
\label{tab:alignment-decomposition}
\centering
\small
\setlength{\tabcolsep}{8pt}
\renewcommand{\arraystretch}{1.08}
\begin{tabular}{lccc}
\toprule
Dataset & Alignment loss & Squared composition residual \(R^2\) & Stratum-gap variance
\tabularnewline
\midrule
Bias in Bios
& $66.5$--$72.6$
& $98.7$--$99.3$
& $60.7$--$69.9$
\tabularnewline

Folktables occupation
& $85.1$--$92.8$
& $75.0$--$94.2$
& $85.9$--$92.6$
\tabularnewline

Folktables income
& $60.0$--$72.5$
& $88.8$--$92.2$
& $55.0$--$69.0$
\tabularnewline

ATUS
& $36.1$--$57.0$
& $59.7$--$60.9$
& $-88.8$--$-5.9$
\tabularnewline
\bottomrule
\end{tabular}
\end{table}

\subsection{Bias in Bios text classification}\label{app:exp-biobias}

The text experiment uses the Hugging Face dataset
\texttt{LabHC/bias\_in\_bios} to predict one of the top 10 occupations:
attorney, dentist, journalist, nurse, photographer, physician, professor, psychologist,
surgeon, and teacher. The comparison groups are Male and Female, and occupation is the reporting stratum for stratum curve summaries.  The true occupation label is the prediction target and is used after prediction to define reporting cells. It is not included in the reporting profile used for standardization, and no realized set size, calibration threshold, or reversal indicator enters the profile.

The DistilBERT classifier is initialized from
\texttt{distilbert-base-uncased} \citep{sanh2019distilbert}.  Bias in Bios uses groupwise conformal thresholds for the delivered sets and reports SIMPLE, APS \citep{romano2020classification}, SAPS \citep{huang2024conformal}, and RAPS \citep{angelopoulos2021uncertainty} on the grid in Table~\ref{tab:appendix-score-grids}. The profile used for standardization consists of fixed probability-profile summaries derived from the classifier output.  The same reporting profile and fit in Table~\ref{tab:appendix-profile-overlap} are used across score variants, so the score specific entries in Table~\ref{tab:score-specific-evidence} compare the same standardization design under different set construction rules.

\begin{table}[t]
\caption{Bias in Bios balance evaluations for the standardized report on the test split.}
\label{tab:appendix-bias-in-bios-diagnostics}
\centering
\footnotesize
\setlength{\tabcolsep}{4pt}
\renewcommand{\arraystretch}{1.12}
\begin{tabularx}{\columnwidth}{@{}>{\raggedright\arraybackslash}X c@{}}
\toprule
Measure & Value \\
\midrule
Effective sample size fraction & $0.833$ \\
Profile mean standardized mean difference & $0.0924\to0.0133$ \\
Profile maximum standardized mean difference & $0.4331\to0.0458$ \\
\bottomrule
\end{tabularx}
\end{table}

All four Bias in Bios score variants follow the same test inclusion rule. We retain an occupation only when both Gender groups contain at least 200 test observations. For each retained occupation, we evaluate every \(\alpha\) value on the prespecified grid, using exactly the same occupation--\(\alpha\) pairs for the raw and standardized summaries. A reversal is counted when \(\Delta(\alpha)\Delta_z(\alpha)<0\), \(|\Delta(\alpha)|\ge 0.01\), and \(|\Delta_z(\alpha)|\ge 0.02\). The conflict strength remains \([-\Delta(\alpha)\Delta_z(\alpha)]_+\), summed over the same supported occupation--\(\alpha\) pairs.

\subsection{Folktables occupation prediction}\label{app:exp-folktables-occupation}

The Folktables occupation task uses ACS person records from 2014 through 2018.  Each year is treated as a separate problem.  The prediction target is top 10 occupation, the comparison group is Gender, and the reporting stratum is education bucket. The classifier is a histogram gradient boosting model and the score is SIMPLE. Education can be a covariate for the classifier, but it is excluded from the profile $H$ used for reporting standardization. This leaves education stratum balance as an evaluation of the standardized report.

The exact yearly split sizes are listed in Table~\ref{tab:appendix-split-sizes}.  The test grid in Table~\ref{tab:appendix-score-grids} uses the resolved support rule of at least 200 test observations in each comparison group within the reporting stratum. The same support rule is applied to raw and standardized summaries before comparing reversal counts or conflict strengths.

\begin{table}[t]
\caption{Folktables occupation evaluation summary across 2014--2018.}
\label{tab:appendix-folktables-occupation-diagnostics}
\centering
\footnotesize
\setlength{\tabcolsep}{4pt}
\renewcommand{\arraystretch}{1.12}
\begin{tabularx}{\columnwidth}{@{}>{\raggedright\arraybackslash}Xcc@{}}
\toprule
Measure & Raw & Standardized \\
\midrule
Reversals, summed & $930$ & $614$ \\
High conflict strength reversals & $400$ & $1$ \\
Conflict strength & $39.664$ & $1.609$ \\
Years with reversal count drop & \multicolumn{2}{c}{$5/5$} \\
Years with strength drop & \multicolumn{2}{c}{$5/5$} \\
Median cell total variation & $0.124$ & $0.045$ \\
Median maximum standardized mean difference & $0.783$ & $0.355$ \\
Effective sample size fraction, range & \multicolumn{2}{c}{$0.611$--$0.702$} \\
Effective sample size fraction, median & \multicolumn{2}{c}{$0.629$} \\
\bottomrule
\end{tabularx}
\end{table}

These results support the tabular replication claim: all five years show lower reversal counts and lower total conflict strength after standardization, while high conflict reversals nearly vanish.

\subsection{Folktables income quintile stress test}\label{app:exp-folktables-income}

The income quintile stress test also uses ACS person records from 2014 through 2018, but changes both the target and comparison group. The target is income quintile, the comparison groups are White and Black, and the reporting strata are education buckets. This design separates supported reversal counts from the magnitude of aggregate and stratum disagreement.

The split roles, score, support rule, and raw to standardized comparison are the same as in Appendix~\ref{app:exp-folktables-occupation}; exact yearly split sizes and grids are listed in Tables~\ref{tab:appendix-split-sizes} and~\ref{tab:appendix-score-grids}. The profile $H$ uses fixed softmax summaries and excludes race and education. The restricted classifier feature policy removes direct income proxy features before prediction and propensity reporting, keeping the standardized comparison tied to prediction difficulty summaries.

\begin{table}[t]
\caption{Folktables income quintile evaluation summary across 2014--2018.}
\label{tab:appendix-folktables-income-diagnostics}
\centering
\footnotesize
\setlength{\tabcolsep}{4pt}
\renewcommand{\arraystretch}{1.12}
\begin{tabularx}{\columnwidth}{@{}>{\raggedright\arraybackslash}Xcc@{}}
\toprule
Measure & Raw & Standardized \\
\midrule
Reversals, summed & $922$ & $717$ \\
High conflict strength reversals & $85$ & $0$ \\
Conflict strength & $11.890$ & $0.509$ \\
Years with reversal count drop & \multicolumn{2}{c}{$4/5$} \\
Years with strength drop & \multicolumn{2}{c}{$5/5$} \\
Median cell total variation & $0.120$ & $0.034$ \\
Median maximum standardized mean difference & $0.290$ & $0.037$ \\
Effective sample size fraction, range & \multicolumn{2}{c}{$0.949$--$0.954$} \\
Effective sample size fraction, median & \multicolumn{2}{c}{$0.950$} \\
\bottomrule
\end{tabularx}
\end{table}

The 2017 row explains why the paper reports more than reversal counts alone.  Supported reversals increase from $182$ to $202$, while high conflict strength reversals change from $27$ to $0$ and conflict strength changes from $2.675$ to $0.144$.  The standardized curve therefore reduces the magnitude of aggregate and stratum disagreement even though one discrete count rises.

\subsection{ATUS activity prediction}\label{app:exp-atus}

The ATUS experiment uses daily time use diary episodes. A closed world top 10 non sleep activity vocabulary is learned from training respondents, and episodes outside that vocabulary are removed in every split. The closed world vocabulary leaves 178,477 analyzed episodes, with 17,783 test observations for final curve evaluation, as shown in Table~\ref{tab:appendix-split-sizes}.  The comparison group is Gender and the reporting stratum is employment status.

The classifier is a regularized histogram gradient boosting classifier trained on time, calendar, demographic, region, and household features available before reporting. The primary score is SIMPLE, with APS \citep{romano2020classification}, RAPS \citep{angelopoulos2021uncertainty}, and SAPS \citep{huang2024conformal} as robustness checks on the grid in Table~\ref{tab:appendix-score-grids}. The standardization profile uses the four fixed probability features in Table~\ref{tab:appendix-profile-overlap}. The summaries report profile balance, employment stratum mixture balance, and effective sample size for the standardized comparison.

\begin{table}[t]
\caption{ATUS evaluation details for the primary SIMPLE score on the test split.}
\label{tab:appendix-atus-diagnostics}
\centering
\footnotesize
\setlength{\tabcolsep}{4pt}
\renewcommand{\arraystretch}{1.12}
\begin{tabularx}{\columnwidth}{@{}>{\raggedright\arraybackslash}Xcc@{}}
\toprule
Measure & Raw & Standardized \\
\midrule
SIMPLE reversals & $36$ & $10$ \\
SIMPLE conflict strength & $0.463730$ & $0.002697$ \\
Maximum profile standardized mean difference & $0.0817$ & $0.0101$ \\
Employment cell total variation & $0.142$ & $0.091$ \\
Effective sample size fraction & \multicolumn{2}{c}{$0.981913$} \\
\bottomrule
\end{tabularx}
\end{table}

\FloatBarrier

\section{Standardization targets and profile choice}\label{app:targets}

The main text uses the standardized profile distribution $Q_H$ because the empirical goal is a stable standardized comparison. This appendix compares the raw, marginal-profile, standardized, and uniform-cell reports using the notation of Appendix~\ref{app:proofs}. Here $P_H$ denotes the marginal distribution of the prediction profile $H=H(X)$ in the evaluation population: $P_H(B)=\mathbb{P}\{H(X)\in B\}$. It is not group-specific; raw group reports average over $P(H\in\cdot\mid G=g)$, the marginal-profile report averages both groups over \(P_H\), and the main standardized report averages both groups over \(Q_H\).

\paragraph{Observed-mixture report.}
The raw curve uses the profile distribution observed within group \(g\):
\begin{equation}
    V_g(\alpha)=
    \mathbb{E}_{H\sim P(H\mid G=g)}
    \mathbb{E}\{M_\alpha(X)\mid G=g,H\}.
\end{equation}
It therefore reports the mean set size delivered to group \(g\) under that group's observed mixture of prediction profiles.

\paragraph{Marginal profile report.}
Another option is to average both groups over the marginal profile distribution \(P_H\):
\begin{equation}
    V_g^{P_H}(\alpha)=
    \mathbb{E}_{H\sim P_H}
    \mathbb{E}\{M_\alpha(X)\mid G=g,H\}.
\end{equation}
In the two-group case with \(e(h)=\mathbb{P}(G=a\mid H=h)\), the corresponding inverse-probability factors are proportional to
\begin{equation}
    \frac{1}{e(H)}
    \qquad\text{for group } a,
    \qquad
    \frac{1}{1-e(H)}
    \qquad\text{for group } b.
\end{equation}
These factors become large when \(e(H)\) is close to \(0\) or \(1\), so the marginal-profile report can be sensitive to weak overlap.

\paragraph{Standardized report.}
The main analysis instead uses \(Q_H\) from \cref{eq:overlap-target}, giving the standardized curve \(\tilde V_g(\alpha)\) in \cref{eq:standardized-curve}. The corresponding factors are
\begin{equation}
    w_a(H)=1-e(H),
    \qquad
    w_b(H)=e(H).
\end{equation}
Both factors are bounded. Profiles dominated by one group therefore receive less weight, and \(\tilde V_g(\alpha)\) compares the groups over profile regions represented in both.

\paragraph{Comparing the reports.}
These contrasts differ in the distribution used to average the same delivered set sizes. The raw gap uses the group-specific distributions \(P(H\mid G=g)\), the marginal-profile contrast
\(V_a^{P_H}(\alpha)-V_b^{P_H}(\alpha)\)
uses \(P_H\), and the standardized contrast
\(\tilde\Delta(\alpha)=\tilde V_a(\alpha)-\tilde V_b(\alpha)\)
uses \(Q_H\). They are therefore different population indices, not competing estimates of the same quantity. A difference between the raw and standardized gaps reflects sensitivity to the reporting population; the predictor, score, calibration threshold, and prediction sets remain unchanged.

\begin{table*}[t]
\caption{Several population indices for the same delivered set sizes. The paper uses the raw and standardized reports because together they show the observed burden and the standardized burden.}
\label{tab:target-comparison}
\centering
\footnotesize
\setlength{\tabcolsep}{3pt}
\renewcommand{\arraystretch}{1.12}
\begin{tabularx}{\textwidth}{@{}p{0.16\textwidth}p{0.18\textwidth}p{0.20\textwidth}X@{}}
\toprule
Report & Target population & Weight behavior & Interpretation \\
\midrule
Raw group curve & Group-specific $P(H\mid G=g)$ & No reweighting & Observed-mixture burden delivered to each group. \\
Marginal-profile curve & Marginal distribution $P_H$ & $1/e(H)$ and $1/\{1-e(H)\}$ & Full-population standardized contrast; sensitive to weak support. \\
Standardized curve & Common standardized distribution $Q_H$ & Bounded factors $1-e(H)$ and $e(H)$ & Propensity standardized contrast; bounded and stable. \\
Uniform-cell curve & Selected reporting cells equally weighted & Cell weights chosen by design & Cell-equal sensitivity summary; can overweight rare or unstable cells. \\
\bottomrule
\end{tabularx}
\end{table*}

\paragraph{Why standardize the prediction profile \(H\) rather than the cells.}
A distribution defined directly over the cells provides a useful alternative summary, but it does not serve the same purpose as propensity standardization. The variable \(Z\) defines where finer-stratum differences are examined, and its meaning and granularity vary across tasks. Matching \(Z\) controls differences in the proportions of the reported strata. Our goal is further to compare set sizes under similar prediction characteristics, which can remain different within the same stratum. In particular, \(H\) is constructed from frozen classifier summaries that are directly relevant to the delivered prediction sets. We use the standardized distribution \(Q_H\) because it emphasizes profile regions represented in both groups. The resulting comparison is based on prediction profiles that are both relevant to the delivered sets and well represented in the two groups, rather than only on matching the proportions of the reported strata.

\paragraph{Reporting cells.}
Standardization balances functions of \(H\), not the reporting variable \(Z\) unless \(Z\) is included in \(H\). In our experiments, \(Z\) is often excluded deliberately because it is the axis on which finer-stratum differences are reported. The remaining difference between \(\tilde\pi_a\) and \(\tilde\pi_b\) therefore records residual reporting-cell composition after profile standardization.

\paragraph{Choosing the prediction profile.}
The profile \(H\) determines which prediction characteristics are balanced before the group curves are compared. We construct \(H\) from frozen classifier summaries such as top probabilities, margins, entropy, and top-label indicators. We do not include realized labels, realized set sizes, calibration thresholds, or final curve summaries. The reporting variable \(Z\) is also excluded when it is the axis of comparison. Support, balance, and effective sample size are then checked for the fitted reweighting.

\paragraph{Finite-sample stabilization.}
The population definitions use the ideal propensity \(e(H)\). In the experiments, the propensity is estimated on the propensity split and then frozen before test evaluation. Standardized curves are computed with normalized H\'ajek means after the prespecified bounded-weight stabilization. Effective sample size and balance checks accompany these estimates to summarize finite-sample support. The reporting weights do not enter the conformal calibration quantile.

\section{Tilted Mondrian details}\label{app:no-free-score-intervention}

\subsection{Tilted Mondrian}
\label{app:tilted-mondrian}

\paragraph{A demonstration example.}

As an auxiliary \(28\) occupation illustration, we finetune DistilBERT for one epoch on all $28$ Bias in Bios occupations, using a $40/25/25/10$ train/calibration/tilt/test split, SIMPLE scores at $\alpha=0.10$, group-Mondrian calibration with $\epsilon=0.01$. The intervention is selected using only $\mathcal D_{\mathrm{tilt}}$, then frozen before recalibration and test evaluation. For the ``model'' occupation, baseline coverage on the tilt split was $37.7\%$ in the Male group and $88.4\%$ in the Female group, a gap of $50.7$ percentage points, with a similarly large test gap of $48.8$ percentage points, $41.1\%$ versus $89.8\%$. For ``paralegal'', Tilted Mondrian increased Male tilt split coverage from $30.3\%$ to $40.9\%$ while Female coverage remained $55.3\%$, reducing the tilt split gap by $10.6$ percentage points. That is, from $25.0$ to $14.4$ percentage points. After recalibration, the test gap decreased from $21.3$ to $0.9$ percentage points, a reduction of $20.4$ percentage points.

The intervention report uses the same prespecified \(G\times Z\) cells as Section~\ref{sec:empirical-protocol-evidence}. Bias in Bios has label-defined \(Z=Y\), whereas Folktables and ATUS use covariate-defined intervention cells. Table~\ref{tab:appendix-score-intervention-cells} lists the group--stratum definitions.

\begin{table}[H]
\caption{Intervention cells for Tilted Mondrian.}
\label{tab:appendix-score-intervention-cells}
\centering
\footnotesize
\setlength{\tabcolsep}{4pt}
\renewcommand{\arraystretch}{1.10}
\begin{tabularx}{\columnwidth}{@{}>{\raggedright\arraybackslash}X>{\raggedright\arraybackslash}X@{}}
\toprule
Experiment & Group and intervention stratum \\
\midrule
Bias in Bios & Gender and occupation, with \(Z=Y\) \\
Folktables occupation & Gender and education bucket \\
Folktables income & Race group and education bucket \\
ATUS & Gender and employment status \\
\bottomrule
\end{tabularx}
\end{table}

The tilt split is used to choose the intervention, which is then frozen. The calibration split then supplies ordinary group Mondrian calibration for the frozen score \(S_\theta\), and the test split supplies the final evaluation. The Bias in Bios RAPS row uses \(\lambda=0.01\) as a separate robustness instance; its Section~\ref{sec:empirical-protocol-evidence} curve report uses \(\lambda=0.2\). The search routine is finite and deterministic rather than exhaustive: it combines coordinate paths, critical threshold values, stratum-wise low dimensional block updates, beam search, and Halton search. Each stratum-wise block treats \((\theta_z^+,\theta_z^-)\) jointly and considers advantaged-side, disadvantaged-side, and joint directional updates. The hyperparameters use grid \(0,0.025,\dots,3,\) beam width \(20\), boundary projection grid with \(41\) points, and \(1000\) Halton candidates.

The feasibility checks in Algorithm~\ref{alg:no-free-score-intervention} define the candidate set for the score search. On the tilt split, a retained directional score shift preserves each group's covered count, satisfies the \(\epsilon\) harm limit for cells that were initially advantaged, maintains coverage in cells that were initially disadvantaged, and leaves every stratum gap weakly smaller. Theorem~\ref{thm:no-free-repair-budget} turns these primitive checks into a tilt split result for the largest single stratum mitigation. The improvement is controlled by the original stratum gap and the same group coverage budget available in cells where that group was initially advantaged. Algorithm~\ref{alg:no-free-score-intervention} chooses \(\theta=(\theta^+,\theta^-)\), forms \(S_\theta\), recalibrates the frozen score, and evaluates the induced prediction sets on the test split.

\begin{algorithm}[H]
\caption{Tilted Mondrian CP with directional score shifting (detailed version of Algorithm~\ref{alg:main-directional-shifting})}
\label{alg:no-free-score-intervention}

\noindent\textbf{Input:}
Base nonconformity score $S_0$; binary group $G$ and intervention stratum $Z$; tilt split $\mathcal D_{\mathrm{tilt}}$, calibration split $\mathcal D_{\mathrm{cal}}$, test split $\mathcal D_{\mathrm{test}}$; nominal level $\alpha$, harm budget $\epsilon$, and fixed group-Mondrian rank/tie rule.

\vspace{0.4ex}
\begin{algorithmic}[1]
\STATE On $\mathcal D_{\mathrm{tilt}}$, use the fixed rank/tie rule to compute baseline group-Mondrian thresholds $q_g^0$, sets $\mathcal C_{\alpha,0}$, and covered counts
\[
    k_g^0=
    \sum_{i\in\mathcal D_{\mathrm{tilt}}:G_i=g}
    \mathbf 1\{Y_i\in \mathcal C_{\alpha,0}(X_i)\}.
\]
\STATE For each audited stratum $z$, compute
\[
    c^0_{g,z}=
    \frac{
    \sum_{i\in\mathcal D_{\mathrm{tilt}}}
    \mathbf 1\{G_i=g,Z_i=z,Y_i\in \mathcal C_{\alpha,0}(X_i)\}
    }{
    \sum_{i\in\mathcal D_{\mathrm{tilt}}}
    \mathbf 1\{G_i=g,Z_i=z\}
    },
\]
set $g^+(z)=\arg\max_g c^0_{g,z}$, $g^-(z)=\arg\min_g c^0_{g,z}$, and $d_z=c^0_{g^+(z),z}-c^0_{g^-(z),z}$.
\STATE For nonnegative $\theta_z^+,\theta_z^-$, define
\[
S_\theta(x,y)=S_0(x,y)+
\begin{cases}
\theta_y^+\mathbf 1\{G=g^+(y)\}-\theta_y^-\mathbf 1\{G=g^-(y)\}, & Z=Y,\\[1ex]
\begin{aligned}
\theta_{Z(x)}^+\mathbf 1\{G=g^+(Z(x))\}{}\\[0ex]
-\theta_{Z(x)}^-\mathbf 1\{G=g^-(Z(x))\},
\end{aligned} & Z=Z(X).
\end{cases}
\]
\STATE Build a finite deterministic candidate set from coordinate paths, critical threshold grids, stratum-wise updates of $(\theta_z^+,\theta_z^-)$, block updates, beam, and Halton search.
\FOR{each candidate $\theta$}
    \STATE Reselect thresholds $q_g^\theta$ on \(\mathcal D_{\mathrm{tilt}}\) using the same rank/tie rule so that every group covered count remains exactly $k_g^0$.
    \STATE Recompute $c^\theta_{g,z}$; set $d_z^\theta=|c^\theta_{a,z}-c^\theta_{b,z}|$ and $m_z^\theta=d_z-d_z^\theta$.
    \STATE Retain $\theta$ iff, for every $z$,
    \[
    c^0_{g^+(z),z}-c^\theta_{g^+(z),z}\le\epsilon,\qquad
    c^\theta_{g^-(z),z}\ge c^0_{g^-(z),z},\qquad
    d_z^\theta\le d_z.
    \]
\ENDFOR
\STATE Select retained $\theta$ lexicographically by smallest worst stratum gap, smallest mean stratum gap, largest total mitigation, smallest directional shift norm, and, only for remaining ties, smallest prespecified average set size cost on the tilt split.
\STATE Record tilt split feasibility checks and the Theorem~\ref{thm:no-free-repair-budget} result for the selected $\theta$.
\STATE Freeze $S_\theta$. On $\mathcal D_{\mathrm{cal}}$, obtain final group-Mondrian thresholds at level $\alpha$.
\STATE Evaluate once on $\mathcal D_{\mathrm{test}}$; report test coverage gaps, harmed/helped cells, set sizes, selected shifts, and achieved mitigations.
\end{algorithmic}

\vspace{0.4ex}
\noindent\textbf{Output:}
Frozen score $S_\theta$, final calibrated group thresholds, tilt split feasibility report, and test coverage/set size report.

\vspace{0.3ex}
\noindent\textbf{Invariant:}
$\mathcal D_{\mathrm{test}}$ enters neither candidate generation, feasibility checking, ranking, nor calibration; it is evaluated once after the score and calibration thresholds are frozen.
\end{algorithm}

\FloatBarrier

\begin{table}[H]
\caption{Test APS, RAPS, and SAPS directional score intervention results. Columns follow Table~\ref{tab:score-intervention-simple}. Harm is the largest tilt split coverage loss among initially advantaged cells. Rounding values may conceal changes at higher precision.}
\label{tab:appendix-score-intervention-remaining}
\centering
\small
\setlength{\tabcolsep}{2.0pt}
\renewcommand{\arraystretch}{1.05}
\begin{tabular*}{\textwidth}{@{\extracolsep{\fill}}llccccc@{}}
\toprule
Experiment & Score & \(L^0_{\max}\to L^\theta_{\max}\) & \(\Delta L_{\max}\) & \(\max_z m_z^\theta\) & Harm & \(\Delta |C|\) \\
\midrule
Bias in Bios & APS & \(0.046{\to}0.043\) & \(0.003\) &  \(0.008\) (nurse) & \(0.0077\) & \(0.0414\) \\
Bias in Bios ($\lambda=0.01$) & RAPS & \(0.036{\to}0.036\) & \(0.000\) &  \(0.006\) (nurse) & \(0.0050\) & \(0.3389\) \\
Bias in Bios & SAPS & \(0.178{\to}0.083\) & \(0.095\) &  \(0.173\) (surgeon) & \(0.0082\) & \(0.0090\) \\
Folktables occupation & APS & \(0.105{\to}0.090\) & \(0.015\) &  \(0.019\) (Bachelor) & \(0.0098\) & \(-0.1008\) \\
Folktables occupation & RAPS & \(0.101{\to}0.084\) & \(0.017\) &  \(0.017\) (Less than high school) & \(0.0098\) & \(-0.0491\) \\
Folktables occupation & SAPS & \(0.028{\to}0.013\) & \(0.015\) &  \(0.022\) (Bachelor) & \(0.0100\) & \(0.0078\) \\
Folktables income & APS & \(0.094{\to}0.065\) & \(0.029\) &  \(0.029\) (Less than high school) & \(0.0096\) & \(-0.0031\) \\
Folktables income & RAPS & \(0.094{\to}0.064\) & \(0.029\) &  \(0.029\) (Less than high school) & \(0.0100\) & \(-0.0039\) \\
Folktables income & SAPS & \(0.020{\to}0.020\) & \(0.000\) &  \(0.009\) (Bachelor) & \(0.0040\) & \(0.0017\) \\
ATUS & APS & \(0.013{\to}0.009\) & \(0.004\) &  \(0.004\) (unemployed) & \(0.0093\) & \(-0.0063\) \\
ATUS & RAPS & \(0.006{\to}0.016\) & \(-0.010\) &  \(0.001\) (employed) & \(0.0073\) & \(0.0017\) \\
ATUS & SAPS & \(0.007{\to}0.036\) & \(-0.028\) &  \(0.002\) (employed) & \(0.0093\) & \(0.0004\) \\
\bottomrule
\end{tabular*}
\end{table}

Because feasibility is enforced on the tilt split, a test stratum gap can increase after the intervention is frozen.

Table~\ref{tab:appendix-no-free-certificates} reports the tilt split result from Theorem~\ref{thm:no-free-repair-budget} for every dataset and score. The active term indicates whether the maximum improvement is controlled by the original stratum gap or by the harm budget. The test column reports the achieved test split mitigation for comparison.

\begin{table}[H]
\caption{Tilt split result for Tilted Mondrian. The Theorem~\ref{thm:no-free-repair-budget} bound is evaluated on the tilt split because the intervention feasibility constraints are checked there. The test column reports the corresponding achieved mitigation on the test split and is not used as a feasibility check.}
\label{tab:appendix-no-free-certificates}
\centering
\footnotesize
\setlength{\tabcolsep}{2.4pt}
\renewcommand{\arraystretch}{1.08}
\begin{tabular*}{\textwidth}{@{\extracolsep{\fill}}llcccc@{}}
\toprule
Dataset
& Score
& \shortstack{Theorem~\ref{thm:no-free-repair-budget}\\bound}
& Active term
& \shortstack{Tilt split\\$\max_z m_z^\theta$}
& \shortstack{Test\\$\max_z m_z^\theta$}
\\
\midrule
Bias in Bios & SIMPLE & 0.156 & gap \(d_z\) & 0.125 & 0.071 \\
 & APS & 0.028 & \(\epsilon\)-budget term & 0.017 & 0.008 \\
 & RAPS ($\lambda=0.01$) & 0.041 & gap \(d_z\) & 0.040 & 0.006 \\
 & SAPS & 0.166 & gap \(d_z\) & 0.164 & 0.173 \\
\addlinespace
Folktables occupation & SIMPLE & 0.037 & gap \(d_z\) & 0.022 & 0.024 \\
 & APS & 0.039 & \(\epsilon\)-budget term & 0.021 & 0.019 \\
 & RAPS & 0.039 & \(\epsilon\)-budget term & 0.015 & 0.017 \\
 & SAPS & 0.035 & gap \(d_z\) & 0.025 & 0.022 \\
\addlinespace
Folktables income & SIMPLE & 0.005 & gap \(d_z\) & 0.005 & 0.009 \\
 & APS & 0.054 & \(\epsilon\)-budget term & 0.020 & 0.029 \\
 & RAPS & 0.054 & \(\epsilon\)-budget term & 0.022 & 0.029 \\
 & SAPS & 0.009 & gap \(d_z\) & 0.009 & 0.009 \\
\addlinespace
ATUS & SIMPLE & 0.009 & gap \(d_z\) & 0.001 & 0.001 \\
 & APS & 0.024 & gap \(d_z\) & 0.020 & 0.004 \\
 & RAPS & 0.017 & gap \(d_z\) & 0.016 & 0.001 \\
 & SAPS & 0.020 & gap \(d_z\) & 0.009 & 0.002 \\
\bottomrule
\end{tabular*}
\end{table}

\begin{table*}[t]
\caption{Sensitivity to the harm budget $\epsilon$.
Each entry reports $B_\epsilon/M_\epsilon$, where $B_\epsilon$ is the
Theorem~\ref{thm:no-free-repair-budget} bound and
$M_\epsilon=\max_z m_z^\theta$ is the achieved tilt split mitigation.
A $\dagger$ indicates that the $\epsilon$-budget term is active in the bound;
otherwise the original stratum gap $d_z$ is active.}
\label{tab:epsilon-sensitivity-theorem6}
\centering
\footnotesize
\setlength{\tabcolsep}{3.5pt}
\renewcommand{\arraystretch}{1.06}
\begin{tabular*}{\textwidth}{@{\extracolsep{\fill}}lccccc@{}}
\toprule
Score
& $\epsilon=.01$
& $\epsilon=.02$
& $\epsilon=.03$
& $\epsilon=.04$
& $\epsilon=.05$ \\
\midrule

\multicolumn{6}{@{}l}{\textbf{Bias in Bios}}\\[-1pt]
SIMPLE
& $0.1557/0.1245$
& $0.1557/0.1197$
& $0.1557/0.1268$
& $0.1557/0.1341$
& $0.1557/0.1363$ \\
APS
& $0.0276^{\dagger}/0.0169$
& $0.0477/0.0169$
& $0.0477/0.0169$
& $0.0477/0.0169$
& $0.0477/0.0169$ \\
RAPS ($\lambda=.01$)
& $0.0414/0.0399$
& $0.0414/0.0399$
& $0.0414/0.0399$
& $0.0414/0.0399$
& $0.0414/0.0399$ \\
SAPS
& $0.1656/0.1640$
& $0.1656/0.1373$
& $0.1656/0.1293$
& $0.1656/0.1628$
& $0.1656/0.1628$ \\

\addlinespace[2pt]
\multicolumn{6}{@{}l}{\textbf{Folktables occupation}}\\[-1pt]
SIMPLE
& $0.0370/0.0225$
& $0.0370/0.0336$
& $0.0370/0.0336$
& $0.0370/0.0336$
& $0.0370/0.0336$ \\
APS
& $0.0395^{\dagger}/0.0211$
& $0.0543^{\dagger}/0.0284$
& $0.0815^{\dagger}/0.0439$
& $0.0963/0.0439$
& $0.0963/0.0483$ \\
RAPS
& $0.0395^{\dagger}/0.0150$
& $0.0543^{\dagger}/0.0280$
& $0.0815^{\dagger}/0.0280$
& $0.0875/0.0366$
& $0.0875/0.0648$ \\
SAPS
& $0.0348/0.0253$
& $0.0348/0.0273$
& $0.0348/0.0282$
& $0.0348/0.0282$
& $0.0348/0.0282$ \\

\addlinespace[2pt]
\multicolumn{6}{@{}l}{\textbf{Folktables income}}\\[-1pt]
SIMPLE
& $0.0054/0.0049$
& $0.0054/0.0049$
& $0.0054/0.0049$
& $0.0054/0.0049$
& $0.0054/0.0049$ \\
APS
& $0.0540^{\dagger}/0.0198$
& $0.0878/0.0309$
& $0.0878/0.0381$
& $0.0878/0.0404$
& $0.0878/0.0404$ \\
RAPS
& $0.0540^{\dagger}/0.0222$
& $0.0878/0.0324$
& $0.0878/0.0371$
& $0.0878/0.0475$
& $0.0878/0.0475$ \\
SAPS
& $0.0092/0.0092$
& $0.0092/0.0092$
& $0.0092/0.0092$
& $0.0092/0.0092$
& $0.0092/0.0092$ \\

\addlinespace[2pt]
\multicolumn{6}{@{}l}{\textbf{ATUS}}\\[-1pt]
SIMPLE
& $0.0092/0.0009$
& $0.0092/0.0009$
& $0.0092/0.0009$
& $0.0092/0.0009$
& $0.0092/0.0009$ \\
APS
& $0.0242/0.0203$
& $0.0242/0.0231$
& $0.0242/0.0231$
& $0.0242/0.0231$
& $0.0242/0.0231$ \\
RAPS
& $0.0170/0.0157$
& $0.0170/0.0157$
& $0.0170/0.0157$
& $0.0170/0.0157$
& $0.0170/0.0157$ \\
SAPS
& $0.0195/0.0091$
& $0.0195/0.0090$
& $0.0195/0.0090$
& $0.0195/0.0090$
& $0.0195/0.0090$ \\
\bottomrule
\end{tabular*}
\end{table*}
Table~\ref{tab:epsilon-sensitivity-theorem6} shows the sensitivity analysis of Tilted Mondrian mitigation limit. When the harm budget \(\epsilon\) is the main factor to limit the mitigation, increasing \(\epsilon\) enlarges the theoretical mitigation limit; once the original stratum gap \(d_z\) becomes the main factor, the bound saturates. Configurations that are already saturated at \(\epsilon=0.01\) therefore remain unchanged across the grid. Across all settings, the mitigation stays within the corresponding theoretical limit.

\paragraph{Stress test under direct finer-cell calibration.}
Direct finer-cell calibration is a natural and direct choice when coverage control is required in every \((G,Z)\) cell. Tilted Mondrian is not intended to replace this construction. It instead acts as an additional score adapter on top of group-Mondrian calibration. In our setting, \(Z\) indexes the finer strata targeted for mitigation, for either coverage or set-size disparities, while the final calibration remains at \(G\). Equalizing coverage does not necessarily equalize set size, and equalizing set size does not necessarily equalize coverage. Tilted Mondrian is used to reduce the finer-stratum disparities while satisfying prespecified constraints. When the finer strata are imbalanced, direct cellwise calibration estimates quantiles from unequal and sometimes small calibration samples, making these quantiles more discrete and more sensitive to sampling variation. This can also yield conservatively large prediction sets in small cells.

We additionally compare the distribution of delivered set sizes under Tilted Mondrian and direct \((G,Z)\)-cell calibration on the same \(20\) resplits for coverage mitigation. Because the procedures calibrate at different granularities, this is a descriptive stress test rather than a like-for-like comparison. Furthermore, direct cellwise Mondrian does not impose the loss-control condition used by Tilted Mondrian. On the test split, we further record the largest cellwise coverage
loss from the group-Mondrian over strata initially identified on \(\mathcal D_{\rm tilt}\). 
Using the \(1\)-percentage-point harm budget, this loss exceeds \(1\) pp in \(289/320\)
resplits under direct cellwise calibration, compared with \(103/320\)
under Tilted Mondrian. Tables~\ref{tab:joint-mondrian-intervention-full} and~\ref{tab:joint-mondrian-rmad-full} report two dispersion measures of the final set sizes, the Gini coefficient and relative mean absolute deviation (RMAD), computed within each test split and then averaged across resplits. Tilted Mondrian yields lower dispersion in \(15/16\) data--score configurations under both measures, and in \(282/320\) and \(296/320\) paired resplits, respectively.

\begin{table}[t]
\centering
\caption{Set size dispersion under direct finer-cell calibration: Gini coefficient.
For each resplit, let $S_i=|\mathcal C(X_i)|$ and
$\bar S=n^{-1}\sum_i S_i$.
We compute
$G=(2n^2\bar S)^{-1}\sum_{i,j}|S_i-S_j|$.
Entries are mean $\pm$ SD over the same 20 paired resplits.
Lower is better; bold marks the lower value.}
\label{tab:joint-mondrian-intervention-full}
\small
\setlength{\tabcolsep}{3pt}
\renewcommand{\arraystretch}{1.0}
\begin{tabular*}{\textwidth}{@{\extracolsep{\fill}}llcccc@{}}
\toprule
Score & Method
& Bias in Bios
& \shortstack{Folktables\\occupation}
& \shortstack{Folktables\\income}
& ATUS \\
\midrule
SIMPLE
& Tilted
& $\mathbf{0.0436\pm0.0032}$
& $\mathbf{0.2579\pm0.0009}$
& $\mathbf{0.1915\pm0.0009}$
& $\mathbf{0.1390\pm0.0007}$ \\
& Cellwise
& $0.1118\pm0.0031$
& $0.2722\pm0.0007$
& $0.1930\pm0.0008$
& $0.1536\pm0.0011$ \\
\addlinespace[1pt]

APS
& Tilted
& $\mathbf{0.3786\pm0.0246}$
& $\mathbf{0.2789\pm0.0012}$
& $\mathbf{0.2259\pm0.0012}$
& $0.1185\pm0.0011$ \\
& Cellwise
& $0.4068\pm0.0010$
& $0.3005\pm0.0012$
& $0.2262\pm0.0010$
& $\mathbf{0.1133\pm0.0015}$ \\
\addlinespace[1pt]

RAPS
& Tilted
& $\mathbf{0.3036\pm0.0153}$
& $\mathbf{0.2391\pm0.0012}$
& $\mathbf{0.2259\pm0.0011}$
& $\mathbf{0.0998\pm0.0009}$ \\
& Cellwise
& $0.3687\pm0.0009$
& $0.2619\pm0.0014$
& $0.2262\pm0.0010$
& $0.1041\pm0.0016$ \\
\addlinespace[1pt]

SAPS
& Tilted
& $\mathbf{0.0434\pm0.0035}$
& $\mathbf{0.2643\pm0.0006}$
& $\mathbf{0.2082\pm0.0009}$
& $\mathbf{0.0380\pm0.0007}$ \\
& Cellwise
& $0.1790\pm0.0039$
& $0.2729\pm0.0008$
& $0.2092\pm0.0008$
& $0.0886\pm0.0014$ \\
\bottomrule
\end{tabular*}
\end{table}

\begin{table}[t]
\centering
\caption{Set size dispersion under direct finer-cell calibration:
relative mean absolute deviation.
For each resplit, let $S_i=|\mathcal C(X_i)|$ and
$\bar S=n^{-1}\sum_i S_i$.
We compute
$\mathrm{RMAD}=(n\bar S)^{-1}\sum_i|S_i-\bar S|$.
Entries are mean $\pm$ SD over the same 20 paired resplits.
Lower is better; bold marks the lower value.}
\label{tab:joint-mondrian-rmad-full}
\small
\setlength{\tabcolsep}{4pt}
\renewcommand{\arraystretch}{1.0}
\begin{tabular*}{\textwidth}{@{\extracolsep{\fill}}llcccc@{}}
\toprule
Score & Method
& Bias in Bios
& \shortstack{Folktables\\occupation}
& \shortstack{Folktables\\income}
& ATUS \\
\midrule
SIMPLE
& Tilted
& $\mathbf{0.068\pm0.007}$
& $\mathbf{0.397\pm0.001}$
& $\mathbf{0.286\pm0.002}$
& $\mathbf{0.203\pm0.001}$ \\
& Cellwise
& $0.209\pm0.006$
& $0.417\pm0.002$
& $0.289\pm0.002$
& $0.222\pm0.002$ \\
\addlinespace[1pt]

APS
& Tilted
& $\mathbf{0.564\pm0.040}$
& $\mathbf{0.411\pm0.002}$
& $\mathbf{0.350\pm0.002}$
& $0.172\pm0.002$ \\
& Cellwise
& $0.609\pm0.002$
& $0.427\pm0.002$
& $0.352\pm0.001$
& $\mathbf{0.165\pm0.002}$ \\
\addlinespace[1pt]

RAPS
& Tilted
& $\mathbf{0.467\pm0.028}$
& $\mathbf{0.372\pm0.002}$
& $\mathbf{0.350\pm0.002}$
& $\mathbf{0.145\pm0.002}$ \\
& Cellwise
& $0.557\pm0.001$
& $0.390\pm0.002$
& $0.352\pm0.001$
& $0.152\pm0.003$ \\
\addlinespace[1pt]

SAPS
& Tilted
& $\mathbf{0.086\pm0.007}$
& $\mathbf{0.419\pm0.002}$
& $\mathbf{0.337\pm0.001}$
& $\mathbf{0.055\pm0.002}$ \\
& Cellwise
& $0.312\pm0.007$
& $0.428\pm0.002$
& $0.339\pm0.001$
& $0.132\pm0.003$ \\
\bottomrule
\end{tabular*}
\end{table}
\subsection{Boosted Tilted Mondrian}
\label{app:boosted-tilted-mondrian}

Motivated by gradient boosted conformal score adaptation~\citep{xie2024boosted}, we consider a covariate adaptive extension of Tilted Mondrian. It keeps the same mitigation target, exact constraints, and recalibration protocol as Algorithm~\ref{alg:no-free-score-intervention}, but replaces the stratum-wise constants $\theta_z^\pm$ by nonnegative functions $\theta_z^\pm(x)$. Using the same stratum convention as above,
\[
S_\theta(x,y)=S_0(x,y)+\theta_z^+(x)\mathbf 1\{G=g^+(z)\}-\theta_z^-(x)\mathbf 1\{G=g^-(z)\}.
\]
Let $\mathcal Z$ denote the strata and $n_{g,z}$ the tilt split size of cell $(g,z)$. Let $\sigma(u)=(1+e^{-u})^{-1}$, and let $\rho_g>0$ be the fixed group-specific smoothing scale. Using the baseline covered count $k_g^0$ defined above, define the soft group threshold $\xi_g(\theta)$ by
\[
\sum_{i:G_i=g}\sigma\!\left(\frac{\xi_g(\theta)-S_\theta(X_i,Y_i)}{\rho_g}\right)=k_g^0.
\]
The corresponding soft cell coverage and gap are
\[
\tilde c_{g,z}(\theta)=\frac{1}{n_{g,z}}\sum_{i:G_i=g,Z_i=z}
\sigma\!\left(\frac{\xi_g(\theta)-S_\theta(X_i,Y_i)}{\rho_g}\right),
\qquad
\tilde d_z(\theta)=\left|\tilde c_{g^+(z),z}(\theta)-\tilde c_{g^-(z),z}(\theta)\right|.
\]
We perform boosting with the log-sum-exp surrogate
\[
\tilde L(\theta)=\frac{1}{50}\log\sum_{z\in\mathcal Z}\exp\!\left\{50\,\tilde d_z(\theta)\right\}.
\]
Thus, every stratum contributes to the gradient, with larger gaps receiving more weight. Let $\mathcal B$ denote the nonnegative amplitude grid. The smooth objective only proposes search directions: every candidate is re-evaluated with the exact tilt split thresholds and the original hard constraints. For the experiments, we use depth-three XGBoost trees, log-sum-exp coefficient \(50\), the amplitude grid \(\mathcal B=\{0\}\cup\{2^{-16},\ldots,2^3\}\), and select among \(0\)–\(10\) boosting rounds by three-fold cross-validation within \(\mathcal D_{\mathrm{tilt}}\).

\begin{algorithm}[H]
\caption{Boosted Tilted Mondrian}
\label{alg:boosted-tilted-mondrian}
\begin{algorithmic}[1]

\STATE \textbf{Input:} Base score $S_0$; group $G$ and stratum $Z$;
tilt, calibration, and test splits
$\mathcal D_{\mathrm{tilt}}$, $\mathcal D_{\mathrm{cal}}$,
and $\mathcal D_{\mathrm{test}}$; level $\alpha$; harm budget $\epsilon$;
fixed smoothing scales $\{\rho_g\}_g$; nonnegative amplitude grid
$\mathcal B$; fixed group-Mondrian rank/tie rule.

\STATE On $\mathcal D_{\mathrm{tilt}}$, compute the baseline covered counts
$k_g^0$, cell coverages $c_{g,z}^0$, mitigation directions $g^\pm(z)$, and
strata $\mathcal Z$.

\STATE Use three-fold cross-validation entirely within
$\mathcal D_{\mathrm{tilt}}$ to select the boosting round $T^\star$.

\STATE Refit on the full $\mathcal D_{\mathrm{tilt}}$. Initialize the
smooth boosting path at $\theta=0$ and the best exact feasible adapter at
$\theta^\star=0$.

\FOR{$t=1,\ldots,T^\star$}

    \STATE Fit depth-three trees $h_{z,t}^{+}(x)$ and $h_{z,t}^{-}(x)$
    to the negative gradients of $\tilde L(\theta)$ with respect to
    $\theta_z^{+}(x)$ and $\theta_z^{-}(x)$.

    \STATE For each $(\lambda^+,\lambda^-)\in\mathcal B^2$, form
    candidates
    \[
    \theta_{z}^{\prime +}(x)
    =
    \bigl[\theta_z^+(x)+\lambda^+h_{z,t}^{+}(x)\bigr]_+,
    \qquad
    \theta_{z}^{\prime -}(x)
    =
    \bigl[\theta_z^-(x)+\lambda^-h_{z,t}^{-}(x)\bigr]_+ .
    \]

    \STATE For every candidate $\theta'$, recompute the exact tilt split
    thresholds under the fixed rank/tie rule and retain only candidates
    satisfying the original Tilted Mondrian constraints on
    $\mathcal D_{\mathrm{tilt}}$.

    \STATE Update $\theta^\star$ to the lexicographically best adapter
    among the current $\theta^\star$ and all retained candidates, using
    the same exact ranking rule as Algorithm~\ref{alg:no-free-score-intervention}.

    \STATE Let
    \(
    \theta_{\mathrm{smooth}}
    \in
    \arg\min_{\theta'\ \mathrm{retained}}
    \tilde L(\theta').
    \)

    \IF{$\tilde L(\theta_{\mathrm{smooth}})
        < \tilde L(\theta)$}
        \STATE $\theta\leftarrow\theta_{\mathrm{smooth}}$.
    \ELSE
        \STATE \textbf{break}.
    \ENDIF

\ENDFOR

\STATE Freeze $S_{\theta^\star}$. On $\mathcal D_{\mathrm{cal}}$, compute
the modified true-label scores and obtain the final group-Mondrian
thresholds $\hat q_g$ at level $\alpha$.

\STATE On $\mathcal D_{\mathrm{test}}$, form
\[
\mathcal C_\alpha(x)
=
\{y:S_{\theta^\star}(x,y)\leq \hat q_{G(x)}\},
\]
and report test coverage gaps and set sizes.

\STATE \textbf{Output:} Frozen boosted score $S_{\theta^\star}$,
calibrated group thresholds $\{\hat q_g\}_g$, and test
coverage/set size report.

\end{algorithmic}
\end{algorithm}

Tables~\ref{tab:btm-simple-robustness}--\ref{tab:btm-paired-mean-gap} compare the Boosted Tilted Mondrian with Tilted Mondrian using the SIMPLE score.

\begin{table}[!t]
\caption{SIMPLE robustness comparison between Tilted Mondrian and Boosted Tilted Mondrian. Tilted Mondrian (TM) is summarized over 20 resplits; Boosted Tilted Mondrian (BTM) over the ten shared resplits. For either method, $\Delta L_{\max}=L_{\max}^{0}-L_{\max}^{\mathrm{final}}$, where $L_{\max}$ is the worst supported fine-grained coverage gap; positive values indicate improvement. Continuous entries are mean $\pm$ SD. Null interventions count the unchanged solution.}
\label{tab:btm-simple-robustness}
\centering
\small
\setlength{\tabcolsep}{1.7pt}
\renewcommand{\arraystretch}{1.05}
\begin{tabular*}{\textwidth}{@{\extracolsep{\fill}}lccccc@{}}
\toprule
Method
& \shortstack{Tilt split worst-gap\\reduction $\Delta L_{\max}$}
& \shortstack{Test worst-gap\\reduction $\Delta L_{\max}$}
& \shortstack{Tilt split \\ maximum harm}
& \shortstack{Test mean\\set size change}
& \shortstack{Null\\interventions} \\
\midrule
\multicolumn{6}{@{}l}{\emph{Bias in Bios}}\\[-1pt]
TM (20)  & $\phantom{-}0.0778\pm0.0397$ & $\phantom{-}0.0680\pm0.0345$ & $0.0090\pm0.0020$ & $\phantom{-}0.0038\pm0.0034$ & $0/20$ \\
BTM (10) & $\phantom{-}0.0145\pm0.0144$ & $\phantom{-}0.0145\pm0.0112$ & $0.0041\pm0.0027$ & $\phantom{-}0.0002\pm0.0003$ & $0/10$ \\
\midrule
\multicolumn{6}{@{}l}{\emph{Folktables occupation}}\\[-1pt]
TM (20)  & $\phantom{-}0.0231\pm0.0025$ & $\phantom{-}0.0209\pm0.0060$ & $0.0099\pm0.0001$ & $\phantom{-}0.0040\pm0.0023$ & $0/20$ \\
BTM (10) & $\phantom{-}0.0099\pm0.0049$ & $\phantom{-}0.0075\pm0.0055$ & $0.0076\pm0.0035$ & $\phantom{-}0.0051\pm0.0058$ & $0/10$ \\
\midrule
\multicolumn{6}{@{}l}{\emph{Folktables income}}\\[-1pt]
TM (20)  & $\phantom{-}0.0157\pm0.0049$ & $-0.0002\pm0.0082$ & $0.0074\pm0.0018$ & $-0.0002\pm0.0015$ & $0/20$ \\
BTM (10) & $\phantom{-}0.0071\pm0.0095$ & $\phantom{-}0.0012\pm0.0027$ & $0.0035\pm0.0035$ & $\phantom{-}0.0010\pm0.0017$ & $3/10$ \\
\midrule
\multicolumn{6}{@{}l}{\emph{American Time Use Survey}}\\[-1pt]
TM (20)  & $\phantom{-}0.0031\pm0.0052$ & $-0.0067\pm0.0119$ & $0.0032\pm0.0043$ & $\phantom{-}0.0015\pm0.0024$ & $9/20$ \\
BTM (10) & $\phantom{-}0.0029\pm0.0064$ & $-0.0032\pm0.0068$ & $0.0011\pm0.0031$ & $\phantom{-}0.0005\pm0.0012$ & $8/10$ \\
\bottomrule
\end{tabular*}
\end{table}

\begin{table}[!t]
\caption{Final mean fine-grained coverage gap on the SIMPLE resplits. Entries are averages across resplits; lower is better. This complements Table~\ref{tab:btm-simple-robustness}, which reports worst-gap reduction, by measuring the average remaining disparity across strata.}
\label{tab:btm-paired-mean-gap}
\centering
\small
\setlength{\tabcolsep}{3.5pt}
\renewcommand{\arraystretch}{1.05}
\begin{tabular*}{\textwidth}{@{\extracolsep{\fill}}lcccc@{}}
\toprule
& \multicolumn{2}{c}{Tilt split final mean gap} & \multicolumn{2}{c}{Test final mean gap} \\
\cmidrule(lr){2-3}\cmidrule(lr){4-5}
Dataset & \shortstack{Tilted\\Mondrian} & \shortstack{Boosted Tilted\\Mondrian} & \shortstack{Tilted\\Mondrian} & \shortstack{Boosted Tilted\\Mondrian} \\
\midrule
Bias in Bios & $\mathbf{0.03291}$ & $0.05111$ & $\mathbf{0.03353}$ & $0.04918$ \\
Folktables occupation & $\mathbf{0.00812}$ & $0.01615$ & $\mathbf{0.00994}$ & $0.01845$ \\
Folktables income & $\mathbf{0.00067}$ & $0.00398$ & $\mathbf{0.01127}$ & $0.01132$ \\
American Time Use Survey & $\mathbf{0.00576}$ & $0.00641$ & $0.01294$ & $\mathbf{0.01047}$ \\
\bottomrule
\end{tabular*}
\end{table}

\noindent
\textbf{Take-away.} The covariate adaptivity of BTM does not translate into stronger mitigation. Tilted Mondrian often achieves larger worst-gap reductions and lower average remaining stratum disparities. This suggests that the structured search is better suited to the constrained mitigation.

\FloatBarrier
\section{Finite Sample Robustness}
\label{app:split-robustness}

To assess finite-sample sensitivity, we freeze each trained predictor and its probability outputs and repeatedly repartition the remaining observations. For standardization, the resplits contain calibration, propensity-estimation, and test data; for mitigation, they contain tilt, calibration, and test data. Across \(50\) standardization resplits, conflict strength decreases in \(899\) of \(900\) configuration-resplit evaluations, with normalized test effective sample size fractions between \(0.601\) and \(0.984\) (Table~\ref{tab:reporting-resplit-robustness}). Across \(20\) mitigation resplits per dataset and score, all \(320\) mitigation problems satisfy the prespecified constraints, although test worst-gap improvements vary by dataset and score (Table~\ref{tab:intervention-resplit-robustness}).

\begin{table}[!t]
\caption{Fixed model robustness of propensity standardization. Each row summarizes 50 calibration/propensity/test resplits. Conflict strength and normalized test effective sample size are reported as mean $\pm$ SD; relative reduction is the median [P5, P95]. D/T/I counts resplits in which standardization decreases, ties, or increases the reversal count.}
\label{tab:reporting-resplit-robustness}
\centering
\small
\setlength{\tabcolsep}{2.1pt}
\renewcommand{\arraystretch}{1.03}
\begin{tabular*}{\textwidth}{@{\extracolsep{\fill}}lcccc@{}}
\toprule
Setting & \begin{tabular}[c]{@{}c@{}}Raw $\to$ standardized\\strength\end{tabular}
& \begin{tabular}[c]{@{}c@{}}Relative reduction\\median [P5, P95] (\%)\end{tabular}
& \begin{tabular}[c]{@{}c@{}}Reversal\\D/T/I\end{tabular}
& Test ESS \\
\midrule
\multicolumn{5}{@{}l}{\emph{Bias in Bios}}\\[-1pt]
SIMPLE & $0.118\pm0.034{\to}0.006\pm0.004$ & $95.3\,[83.6,99.5]$ & $49/0/1$ & $0.824\pm0.004$ \\
APS & $0.152\pm0.068{\to}0.012\pm0.008$ & $92.5\,[77.1,98.0]$ & $48/1/1$ & $0.824\pm0.004$ \\
SAPS   & $0.151\pm0.038{\to}0.006\pm0.004$ & $96.4\,[90.2,98.7]$ & $50/0/0$ & $0.824\pm0.004$ \\
RAPS   & $0.156\pm0.056{\to}0.010\pm0.006$ & $93.3\,[81.6,98.4]$ & $49/1/0$ & $0.824\pm0.004$ \\
\midrule
\multicolumn{5}{@{}l}{\emph{Folktables occupation}}\\[-1pt]
2014 & $9.610\pm0.535{\to}0.121\pm0.071$ & $99.0\,[97.5,99.6]$ & $50/0/0$ & $0.666\pm0.003$ \\
2015 & $8.054\pm0.463{\to}0.419\pm0.212$ & $95.4\,[90.4,98.6]$ & $44/0/6$ & $0.609\pm0.004$ \\
2016 & $6.883\pm0.468{\to}0.264\pm0.121$ & $96.3\,[92.9,98.7]$ & $49/0/1$ & $0.696\pm0.007$ \\
2017 & $6.616\pm0.350{\to}0.669\pm0.184$ & $90.0\,[83.2,94.7]$ & $50/0/0$ & $0.640\pm0.005$ \\
2018 & $7.625\pm0.537{\to}0.483\pm0.159$ & $93.7\,[89.8,97.1]$ & $50/0/0$ & $0.619\pm0.005$ \\
\midrule
\multicolumn{5}{@{}l}{\emph{Folktables income}}\\[-1pt]
2014 & $1.665\pm0.475{\to}0.072\pm0.039$ & $96.0\,[91.1,98.7]$ & $41/0/9$ & $0.950\pm0.002$ \\
2015 & $1.800\pm0.593{\to}0.057\pm0.039$ & $96.9\,[93.5,99.0]$ & $39/1/10$ & $0.954\pm0.001$ \\
2016 & $2.145\pm0.551{\to}0.130\pm0.069$ & $94.2\,[88.3,96.7]$ & $43/0/7$ & $0.953\pm0.001$ \\
2017 & $2.317\pm0.607{\to}0.132\pm0.064$ & $94.9\,[88.3,97.6]$ & $33/1/16$ & $0.950\pm0.001$ \\
2018 & $2.208\pm0.610{\to}0.145\pm0.073$ & $93.9\,[88.8,96.0]$ & $42/0/8$ & $0.949\pm0.001$ \\
\midrule
\multicolumn{5}{@{}l}{\emph{ATUS}}\\[-1pt]
SIMPLE & $0.280\pm0.077{\to}0.015\pm0.025$ & $97.8\,[79.3,99.2]$ & $47/1/2$ & $0.982\pm0.001$ \\
APS    & $1.372\pm0.201{\to}0.198\pm0.149$ & $87.3\,[69.4,96.6]$ & $32/16/2$ & $0.982\pm0.001$ \\
RAPS   & $1.016\pm0.152{\to}0.138\pm0.110$ & $88.7\,[70.6,97.0]$ & $32/18/0$ & $0.982\pm0.001$ \\
SAPS   & $0.029\pm0.018{\to}0.014\pm0.014$ & $61.7\,[15.6,98.4]$ & $47/2/1$ & $0.982\pm0.001$ \\
\bottomrule
\end{tabular*}
\end{table}

\FloatBarrier

\begin{table}[!t]
\caption{Fixed-model robustness of Tilted Mondrian for coverage. Each row summarizes 20 tilt/calibration/test resplits; continuous entries are mean $\pm$ SD. \(L^0_{\max}\), \(L^\theta_{\max}\), \(\Delta L_{\max}\), and \(\Delta|C|\) are evaluated on the test split. Harm is the largest tilt split coverage loss among initially advantaged cells. P/T/N counts positive, numerically tied, or negative test worst-gap reductions. All 320 mitigation problems satisfy the prespecified feasibility constraints.}
\label{tab:intervention-resplit-robustness}
\centering
\small
\setlength{\tabcolsep}{1.0pt}
\renewcommand{\arraystretch}{1.03}
\begin{tabular*}{\textwidth}{@{\extracolsep{\fill}}lccccc@{}}
\toprule
Score & $L^0_{\max}{\to} L^\theta_{\max}$ & $\Delta L_{\max}$ & Harm & $\Delta|C|$ & P/T/N \\
\midrule
\multicolumn{6}{@{}l}{\emph{Bias in Bios}}\\[-1pt]
SIMPLE & $0.172\pm0.031{\to}0.104\pm0.031$ & $\phantom{-}0.0680\pm0.0345$ & $0.0090\pm0.0020$ & $\phantom{-}0.0038\pm0.0034$ & $20/0/0$ \\
APS    & $0.043\pm0.007{\to}0.040\pm0.005$ & $\phantom{-}0.0025\pm0.0030$ & $0.0036\pm0.0022$ & $\phantom{-}0.2955\pm0.2220$ & $19/1/0$ \\
RAPS   & $0.041\pm0.009{\to}0.034\pm0.007$ & $\phantom{-}0.0074\pm0.0082$ & $0.0046\pm0.0012$ & $\phantom{-}0.4799\pm0.0959$ & $12/7/1$ \\
SAPS   & $0.179\pm0.034{\to}0.084\pm0.016$ & $\phantom{-}0.0951\pm0.0328$ & $0.0093\pm0.0011$ & $\phantom{-}0.0071\pm0.0036$ & $20/0/0$ \\
\midrule
\multicolumn{6}{@{}l}{\emph{Folktables occupation}}\\[-1pt]
SIMPLE & $0.044\pm0.007{\to}0.023\pm0.006$ & $\phantom{-}0.0209\pm0.0060$ & $0.0099\pm0.0001$ & $\phantom{-}0.0040\pm0.0023$ & $20/0/0$ \\
APS    & $0.103\pm0.007{\to}0.089\pm0.008$ & $\phantom{-}0.0147\pm0.0026$ & $0.0093\pm0.0014$ & $-0.0915\pm0.0128$ & $20/0/0$ \\
RAPS   & $0.096\pm0.007{\to}0.081\pm0.007$ & $\phantom{-}0.0150\pm0.0037$ & $0.0098\pm0.0002$ & $-0.0455\pm0.0063$ & $20/0/0$ \\
SAPS   & $0.043\pm0.008{\to}0.023\pm0.009$ & $\phantom{-}0.0197\pm0.0074$ & $0.0091\pm0.0015$ & $\phantom{-}0.0057\pm0.0033$ & $20/0/0$ \\
\midrule
\multicolumn{6}{@{}l}{\emph{Folktables income}}\\[-1pt]
SIMPLE & $0.020\pm0.006{\to}0.020\pm0.010$ & $-0.0002\pm0.0082$ & $0.0074\pm0.0018$ & $-0.0003\pm0.0015$ & $10/1/9$ \\
APS    & $0.091\pm0.017{\to}0.069\pm0.021$ & $\phantom{-}0.0213\pm0.0117$ & $0.0088\pm0.0028$ & $-0.0053\pm0.0051$ & $20/0/0$ \\
RAPS   & $0.091\pm0.017{\to}0.068\pm0.021$ & $\phantom{-}0.0229\pm0.0104$ & $0.0096\pm0.0018$ & $-0.0057\pm0.0049$ & $20/0/0$ \\
SAPS   & $0.021\pm0.009{\to}0.023\pm0.012$ & $-0.0024\pm0.0067$ & $0.0063\pm0.0022$ & $\phantom{-}0.0008\pm0.0011$ & $7/1/12$ \\
\midrule
\multicolumn{6}{@{}l}{\emph{ATUS}}\\[-1pt]
SIMPLE & $0.019\pm0.011{\to}0.026\pm0.017$ & $-0.0067\pm0.0119$ & $0.0033\pm0.0043$ & $\phantom{-}0.0015\pm0.0024$ & $2/10/8$ \\
APS    & $0.026\pm0.015{\to}0.028\pm0.015$ & $-0.0018\pm0.0136$ & $0.0064\pm0.0038$ & $-0.0048\pm0.0046$ & $8/5/7$ \\
RAPS   & $0.020\pm0.014{\to}0.025\pm0.017$ & $-0.0050\pm0.0133$ & $0.0052\pm0.0035$ & $-0.0005\pm0.0035$ & $6/3/11$ \\
SAPS   & $0.020\pm0.007{\to}0.027\pm0.016$ & $-0.0071\pm0.0145$ & $0.0040\pm0.0039$ & $\phantom{-}0.0009\pm0.0016$ & $6/8/6$ \\
\bottomrule
\end{tabular*}
\end{table}

\FloatBarrier
\section{Tilted Mondrian for Set Size and Fundamental Tradeoffs}\label{app:alternative-repair-target}

We use coverage as the primary mitigation target in the main text because it is the reliability quantity guaranteed by CP. Here we apply the same Tilted Mondrian construction to set size and report the resulting coverage changes. For cell \((g,z)\), let \(\bar M_{g,z}\) denote its mean prediction-set size. The baseline set-size gap in stratum \(z\) is
\[
L^0_{z,S}=|\bar M^0_{a,z}-\bar M^0_{b,z}|,
\]
and after a score adjustment \(\theta\) it is
\[
L^\theta_{z,S}=|\bar M^\theta_{a,z}-\bar M^\theta_{b,z}|.
\]
The set-size-targeted search uses these gaps in place of the coverage gaps \(d_z\) and \(d_z^\theta\), with the same tilt/calibration/test separation and group-coverage constraint. Harm is the largest set-size increase among cells that initially have the smaller mean set size; the harm budget is \(0.02\). For Folktables income, \(K=5\) and deterministic RAPS uses \(k_{\mathrm{reg}}=5\). Its regularization term is therefore zero at every label rank, so RAPS and APS coincide in this experiment.

\FloatBarrier
\begin{table}[H]
\caption{Fixed-model robustness of set-size-targeted Tilted Mondrian. Each row summarizes 20 tilt/calibration/test resplits; continuous entries are mean $\pm$ SD. Set-size gaps and \(\Delta\mathrm{Cov.}\) are evaluated on the test split. Harm is the largest tilt split set-size increase among initially smaller set cells. P/T/N counts positive, tied, or negative test worst-gap reductions.}
\label{tab:set-size-repair-resplit-robustness}
\centering
\small
\setlength{\tabcolsep}{1.0pt}
\renewcommand{\arraystretch}{1.03}
\begin{tabular*}{\textwidth}{@{\extracolsep{\fill}}lccccc@{}}
\toprule
Score & $L^0_{\max,S}{\to} L^\theta_{\max,S}$ & $\Delta L_{\max,S}$ & Harm & $\Delta\mathrm{Cov.}$ (pp) & P/T/N \\
\midrule
\multicolumn{6}{@{}l}{\emph{Bias in Bios}}\\[-1pt]
SIMPLE & $0.066\pm0.024{\to}0.065\pm0.025$ & $\phantom{-}0.0013\pm0.0039$ & $0.0038\pm0.0059$ & $-0.001\pm0.013$ & $5/12/3$ \\
APS    & $1.269\pm0.054{\to}0.514\pm0.058$ & $\phantom{-}0.7551\pm0.0790$ & $0.0061\pm0.0064$ & $-0.019\pm0.092$ & $20/0/0$ \\
RAPS   & $1.163\pm0.041{\to}0.515\pm0.050$ & $\phantom{-}0.6480\pm0.0573$ & $0.0052\pm0.0083$ & $-0.028\pm0.089$ & $20/0/0$ \\
SAPS   & $0.052\pm0.016{\to}0.054\pm0.032$ & $-0.0019\pm0.0301$ & $0.0018\pm0.0022$ & $-0.001\pm0.021$ & $11/5/4$ \\
\midrule
\multicolumn{6}{@{}l}{\emph{Folktables occupation}}\\[-1pt]
SIMPLE & $0.217\pm0.022{\to}0.168\pm0.022$ & $\phantom{-}0.0490\pm0.0111$ & $0.0198\pm0.0002$ & $\phantom{-}0.009\pm0.022$ & $20/0/0$ \\
APS    & $2.238\pm0.022{\to}2.233\pm0.022$ & $\phantom{-}0.0044\pm0.0040$ & $0.0004\pm0.0006$ & $\phantom{-}0.000\pm0.002$ & $17/2/1$ \\
RAPS   & $1.621\pm0.040{\to}1.580\pm0.044$ & $\phantom{-}0.0411\pm0.0313$ & $0.0159\pm0.0043$ & $\phantom{-}0.008\pm0.019$ & $20/0/0$ \\
SAPS   & $0.273\pm0.027{\to}0.217\pm0.032$ & $\phantom{-}0.0560\pm0.0148$ & $0.0193\pm0.0009$ & $\phantom{-}0.006\pm0.022$ & $20/0/0$ \\
\midrule
\multicolumn{6}{@{}l}{\emph{Folktables income}}\\[-1pt]
SIMPLE & $0.183\pm0.041{\to}0.136\pm0.045$ & $\phantom{-}0.0467\pm0.0256$ & $0.0150\pm0.0070$ & $\phantom{-}0.000\pm0.019$ & $20/0/0$ \\
APS    & $0.201\pm0.036{\to}0.188\pm0.042$ & $\phantom{-}0.0132\pm0.0229$ & $0.0047\pm0.0078$ & $\phantom{-}0.002\pm0.016$ & $10/8/2$ \\
RAPS   & $0.201\pm0.036{\to}0.188\pm0.042$ & $\phantom{-}0.0132\pm0.0229$ & $0.0047\pm0.0078$ & $\phantom{-}0.002\pm0.016$ & $10/8/2$ \\
SAPS   & $0.229\pm0.051{\to}0.190\pm0.049$ & $\phantom{-}0.0394\pm0.0198$ & $0.0177\pm0.0049$ & $-0.007\pm0.026$ & $20/0/0$ \\
\midrule
\multicolumn{6}{@{}l}{\emph{ATUS}}\\[-1pt]
SIMPLE & $0.306\pm0.065{\to}0.303\pm0.061$ & $\phantom{-}0.0030\pm0.0116$ & $0.0021\pm0.0021$ & $-0.001\pm0.006$ & $7/11/2$ \\
APS    & $0.565\pm0.071{\to}0.552\pm0.069$ & $\phantom{-}0.0128\pm0.0123$ & $0.0055\pm0.0069$ & $\phantom{-}0.000\pm0.005$ & $17/2/1$ \\
RAPS   & $0.478\pm0.051{\to}0.472\pm0.048$ & $\phantom{-}0.0063\pm0.0165$ & $0.0056\pm0.0053$ & $\phantom{-}0.000\pm0.011$ & $9/7/4$ \\
SAPS   & $0.181\pm0.043{\to}0.179\pm0.043$ & $\phantom{-}0.0024\pm0.0068$ & $0.0061\pm0.0068$ & $\phantom{-}0.000\pm0.009$ & $8/7/5$ \\
\bottomrule
\end{tabular*}
\end{table}
\FloatBarrier

\noindent\textbf{Take-away.}
The set-size-targeted variant is feasible under the prespecified constraints. On the test split, the worst set-size gap improves, ties, or worsens in \(234\), \(62\), and \(24\) of the \(320\) resplits, respectively.

\begin{table}[H]
\caption{Cross-metric changes under the two Tilted Mondrian targets. Entries are mean $\pm$ SD over 20 fixed-model resplits. Under coverage-targeted mitigation, we report the overall and maximum-cell set-size changes. Under set-size-targeted mitigation, we report the overall and maximum-cell coverage changes in percentage points.}
\label{tab:cross-metric-redistribution}
\centering
\small
\setlength{\tabcolsep}{0.8pt}
\renewcommand{\arraystretch}{1.03}
\begin{tabular*}{\textwidth}{@{\extracolsep{\fill}}lcccc@{}}
\toprule
& \multicolumn{2}{c}{Coverage-targeted mitigation} & \multicolumn{2}{c}{Set-size-targeted mitigation}\\
\cmidrule(lr){2-3}\cmidrule(lr){4-5}
Score & $\Delta\bar M$ & $\max_{g,z}|\Delta\bar M_{g,z}|$ & $\Delta\mathrm{Cov.}$ (pp) & $\max_{g,z}|\Delta c_{g,z}|$ (pp)\\
\midrule
\multicolumn{5}{@{}l}{\emph{Bias in Bios}}\\[-1pt]
SIMPLE & $\phantom{-}0.0038\pm0.0034$ & $0.104\pm0.068$ & $-0.001\pm0.013$ & $0.352\pm0.551$ \\
APS    & $\phantom{-}0.2955\pm0.2220$ & $0.887\pm0.443$ & $-0.019\pm0.092$ & $18.786\pm10.999$ \\
RAPS   & $\phantom{-}0.4799\pm0.0959$ & $1.255\pm0.266$ & $-0.028\pm0.089$ & $12.836\pm9.650$ \\
SAPS   & $\phantom{-}0.0071\pm0.0036$ & $0.137\pm0.050$ & $-0.001\pm0.021$ & $1.601\pm3.629$ \\
\midrule
\multicolumn{5}{@{}l}{\emph{Folktables occupation}}\\[-1pt]
SIMPLE & $\phantom{-}0.0040\pm0.0023$ & $0.162\pm0.033$ & $\phantom{-}0.009\pm0.022$ & $0.399\pm0.101$ \\
APS    & $-0.0915\pm0.0128$ & $0.548\pm0.059$ & $\phantom{-}0.000\pm0.002$ & $0.019\pm0.045$ \\
RAPS   & $-0.0455\pm0.0063$ & $0.391\pm0.040$ & $\phantom{-}0.008\pm0.019$ & $0.475\pm0.115$ \\
SAPS   & $\phantom{-}0.0057\pm0.0033$ & $0.177\pm0.060$ & $\phantom{-}0.006\pm0.022$ & $0.485\pm0.166$ \\
\midrule
\multicolumn{5}{@{}l}{\emph{Folktables income}}\\[-1pt]
SIMPLE & $-0.0002\pm0.0015$ & $0.072\pm0.031$ & $\phantom{-}0.000\pm0.019$ & $0.711\pm0.318$ \\
APS    & $-0.0053\pm0.0051$ & $0.066\pm0.047$ & $\phantom{-}0.002\pm0.016$ & $0.314\pm0.556$ \\
RAPS   & $-0.0057\pm0.0049$ & $0.055\pm0.028$ & $\phantom{-}0.002\pm0.016$ & $0.314\pm0.556$ \\
SAPS   & $\phantom{-}0.0008\pm0.0011$ & $0.054\pm0.023$ & $-0.007\pm0.026$ & $0.707\pm0.307$ \\
\midrule
\multicolumn{5}{@{}l}{\emph{ATUS}}\\[-1pt]
SIMPLE & $\phantom{-}0.0015\pm0.0024$ & $0.149\pm0.256$ & $-0.001\pm0.006$ & $0.039\pm0.080$ \\
APS    & $-0.0048\pm0.0046$ & $0.171\pm0.139$ & $\phantom{-}0.000\pm0.005$ & $0.111\pm0.211$ \\
RAPS   & $-0.0005\pm0.0035$ & $0.160\pm0.139$ & $\phantom{-}0.000\pm0.011$ & $0.111\pm0.206$ \\
SAPS   & $\phantom{-}0.0009\pm0.0016$ & $0.163\pm0.204$ & $\phantom{-}0.000\pm0.009$ & $0.073\pm0.182$ \\
\bottomrule
\end{tabular*}
\end{table}
\FloatBarrier

\noindent\textbf{Take-away.} Aggregate stability does not imply cellwise stability. Either mitigation target can redistribute the other quantity across finer strata, highlighting the fundamental trade-off between coverage and set sizes.

\begin{table}[H]
\caption{Test set-size absolute change under both mitigation targets. For each resplit, we compute the median and maximum of $|\bar M^\theta_{g,z}-\bar M^0_{g,z}|$ over cells; entries are mean $\pm$ SD across 20 resplits.}
\label{tab:repair-target-size-movement}
\centering
\small
\setlength{\tabcolsep}{0.8pt}
\renewcommand{\arraystretch}{1.03}
\begin{tabular*}{\textwidth}{@{\extracolsep{\fill}}lcccc@{}}
\toprule
& \multicolumn{2}{c}{Coverage-targeted mitigation} & \multicolumn{2}{c}{Set-size-targeted mitigation}\\
\cmidrule(lr){2-3}\cmidrule(lr){4-5}
Score & median cell $|\Delta\bar M|$ & max cell $|\Delta\bar M|$ & median cell $|\Delta\bar M|$ & max cell $|\Delta\bar M|$\\
\midrule
\multicolumn{5}{@{}l}{\emph{Bias in Bios}}\\[-1pt]
SIMPLE & $0.0048\pm0.0026$ & $0.104\pm0.068$ & $0.0000\pm0.0001$ & $0.005\pm0.007$ \\
APS    & $0.1407\pm0.1041$ & $0.887\pm0.443$ & $0.0720\pm0.0391$ & $1.191\pm0.116$ \\
RAPS   & $0.1920\pm0.1535$ & $1.255\pm0.266$ & $0.0535\pm0.0253$ & $0.939\pm0.117$ \\
SAPS   & $0.0074\pm0.0035$ & $0.137\pm0.050$ & $0.0007\pm0.0011$ & $0.016\pm0.035$ \\
\midrule
\multicolumn{5}{@{}l}{\emph{Folktables occupation}}\\[-1pt]
SIMPLE & $0.0788\pm0.0153$ & $0.162\pm0.033$ & $0.0115\pm0.0016$ & $0.032\pm0.006$ \\
APS    & $0.0625\pm0.0121$ & $0.548\pm0.059$ & $0.0001\pm0.0002$ & $0.005\pm0.004$ \\
RAPS   & $0.0394\pm0.0083$ & $0.391\pm0.040$ & $0.0100\pm0.0060$ & $0.066\pm0.018$ \\
SAPS   & $0.0752\pm0.0190$ & $0.177\pm0.060$ & $0.0124\pm0.0030$ & $0.042\pm0.012$ \\
\midrule
\multicolumn{5}{@{}l}{\emph{Folktables income}}\\[-1pt]
SIMPLE & $0.0199\pm0.0100$ & $0.072\pm0.031$ & $0.0093\pm0.0053$ & $0.041\pm0.020$ \\
APS    & $0.0176\pm0.0104$ & $0.066\pm0.047$ & $0.0036\pm0.0073$ & $0.028\pm0.047$ \\
RAPS   & $0.0170\pm0.0064$ & $0.055\pm0.028$ & $0.0036\pm0.0073$ & $0.028\pm0.047$ \\
SAPS   & $0.0180\pm0.0075$ & $0.054\pm0.023$ & $0.0112\pm0.0041$ & $0.045\pm0.017$ \\
\midrule
\multicolumn{5}{@{}l}{\emph{ATUS}}\\[-1pt]
SIMPLE & $0.0033\pm0.0033$ & $0.149\pm0.256$ & $0.0002\pm0.0004$ & $0.006\pm0.012$ \\
APS    & $0.0319\pm0.0285$ & $0.171\pm0.139$ & $0.0004\pm0.0007$ & $0.021\pm0.019$ \\
RAPS   & $0.0201\pm0.0197$ & $0.160\pm0.139$ & $0.0007\pm0.0011$ & $0.015\pm0.018$ \\
SAPS   & $0.0058\pm0.0067$ & $0.163\pm0.204$ & $0.0004\pm0.0009$ & $0.010\pm0.009$ \\
\bottomrule
\end{tabular*}
\end{table}
\FloatBarrier

\noindent\textbf{Take-away.} Tilted Mondrian for set sizes is typically more size-conservative, but coverage is less protected: the worst stratum coverage gap worsens in $130/320$ set-size-targeted mitigations versus $54/320$ coverage-targeted mitigations, with $25/320$ versus $0/320$ worsenings exceeding $5$ percentage points.

\end{document}